\documentclass{article}
\usepackage[preprint]{iclr2027_conference}
\usepackage{times}

\usepackage{amsmath,amssymb,amsthm,latexsym,mathrsfs,stmaryrd}
\newtheorem{proposition}{Proposition}

\usepackage[table]{xcolor}
\usepackage{graphicx,epsfig}
\usepackage{booktabs,tabularx,array,multirow,makecell,hhline}
\usepackage[flushleft]{threeparttable}
\usepackage{float,wrapfig,placeins}
\usepackage{subcaption} % Also provides \captionof.
\usepackage[export]{adjustbox}

\usepackage[color,matrix,arrow,all]{xy}
\usepackage{tikz}
\usetikzlibrary{shapes,calc,decorations.pathmorphing}
\usepackage[most]{tcolorbox}
\usepackage{framed,changepage}
\usepackage{algorithm,algpseudocode}
\usepackage{listings}
\usepackage{paralist,enumerate,enumitem}
\usepackage{xspace,xparse,comment,pifont,balance}

\usepackage{url}
\usepackage{hyperref}
\hypersetup{hidelinks}

\let\cite\citep

\newcommand{\method}{GIST}
\newcommand{\methodfull}{\underline{G}enerator-state-\underline{I}nformed \underline{S}upervision for \underline{T}FMs}
\definecolor{gistImproved}{RGB}{0,47,167}
\definecolor{gistNotImproved}{HTML}{008C95}

\newcommand{\stitle}[1]{\par\noindent\textbf{#1}\enspace\ignorespaces}

\newcommand{\eat}[1]{}

\definecolor{shadecolor}{RGB}{220,220,220}

\definecolor{inputcolor}{RGB}{255,139,35}
\definecolor{outputcolor}{RGB}{120,212,252}
\definecolor{embedcolor}{RGB}{254,127,156}
\definecolor{maskcolor}{RGB}{122,128,255}
\definecolor{ecolor}{RGB}{58,149,54}

\definecolor{highcolor}{RGB}{255,153,153}
\definecolor{midcolor}{RGB}{255,204,204}
\definecolor{lowcolor}{RGB}{204,229,255}

\definecolor{green}{RGB}{0,128,0}

\definecolor{yellow}{RGB}{255,200,18}

\newcommand{\be}{\begin{enumerate}}
\newcommand{\ee}{\end{enumerate}}
\newcommand{\beqn}{\begin{eqnarray*}}
\newcommand{\eeqn}{\end{eqnarray*}}

\makeatletter
    \newcommand\figcaption{\def\@captype{figure}\caption}
    \newcommand\tabcaption{\def\@captype{table}\caption}
\makeatother

\tikzstyle{mybox} = [draw=black, fill=black!5, thick,
   rectangle, rounded corners, inner sep=0pt, inner ysep=6pt]
\tikzstyle{fancytitle} =[fill=black, text=white]

\newtcolorbox{findingbox}[2][]{
  colback=gray!10!white,
  colframe=gray!30!white,
  boxrule=0.2mm,
  left=0mm,
  right=0mm,
  top=0mm,
  bottom=0mm,
  coltitle=red!70!black,
  title={#2},
  #1
}

\definecolor{InsightBlue}{RGB}{60, 96, 170}
\definecolor{InsightBg}{RGB}{245, 248, 255}
\definecolor{InsightBorder}{RGB}{180, 198, 231}

\newtcolorbox{insightbox}{
  enhanced,
  colback=InsightBg,
  colframe=InsightBorder,
  boxrule=0.5pt,
  arc=2.5pt,
  left=2pt,
  right=2pt,
  top=1pt,
  bottom=1pt,
  borderline west={2pt}{0pt}{InsightBlue},
  before skip=4pt,
  after skip=6pt
}

\definecolor{sqlkeyword}{rgb}{0, 0, 0.8}   % Blue for keywords
\definecolor{sqlstring}{rgb}{0.8, 0.4, 0}    % Orange for strings
\definecolor{sqlcomment}{rgb}{0, 0.5, 0}   % Green for comments

\lstdefinestyle{sqlstyle}{
  language=SQL,
  basicstyle=\ttfamily\scriptsize,
  keywordstyle=\color{sqlkeyword}\bfseries,
  stringstyle=\color{sqlstring},
  commentstyle=\color{sqlcomment}\itshape,
  breaklines=true,
  breakindent=0pt,
}

\algrenewcommand{\algorithmiccomment}[1]{\bgroup\color{gray}\textit{// #1}\egroup}
\algrenewcommand{\algorithmicrequire}{\textbf{Input:}}
\algrenewcommand{\algorithmicensure}{\textbf{Output:}}
\algnewcommand{\Func}[1]{\mathtt{#1}}

\definecolor{groupgray}{RGB}{245,245,245}
\definecolor{oursgreen}{RGB}{226,239,218}

\title{\method{}: Generator-state-Informed Supervision for Tabular Foundation Models}
\newcommand{\authorAffiliation}{The Hong Kong University of Science and Technology (Guangzhou)}
\newcommand{\authorAddress}{Guangzhou, China}

\author{%
Zhangyang Peng\textsuperscript{1}\quad
Jingze Shi\textsuperscript{1}\quad
Tongyi Zhang\textsuperscript{1}\quad
Yuyu Luo\textsuperscript{1}%
}

\makeatletter
\g@addto@macro\@thanks{%
    \begingroup
    \footnotetext[1]{%
        \authorAffiliation, \authorAddress.
    }%
    \endgroup
}
\makeatother

\hypersetup{pdfauthor={Zhangyang Peng, Jingze Shi, Tongyi Zhang, Yuyu Luo}}

\begin{document}
\maketitle

\begin{abstract}
Tabular foundation models (TFMs) based on prior-data fitted networks typically use sampled hard query labels, leaving avoidable sampling variability in pretraining. A key observation is that intermediate generator states can reveal label uncertainty beyond a sampled outcome. We propose \methodfull{} (\method{}), which uses these states to average query losses over conditional label distributions.
\method{} jointly designs task generation and query supervision, using entropy-calibrated encodings to obtain tractable, input-consistent targets on supported label-dependent branches. We show that this averaging preserves the population objective and yields no larger gradient covariance at fixed parameters within each generating distribution. Under matched pretraining schedules, \method{} improves mean rank across three architectures on both TabArena and TALENT without changing architecture or inference. Across these benchmarks and architectures, \method{} wins 62.1\% of paired dataset-level comparisons against the corresponding hard-label controls.
\end{abstract}

% !TeX root = ../simrb.tex
\section{Introduction}
\label{sec:introduction}

\suppressfloats[t]
% !TeX root = ../simrb.tex
\begin{figure}[t]
    \centering
    \includegraphics[width=\linewidth]{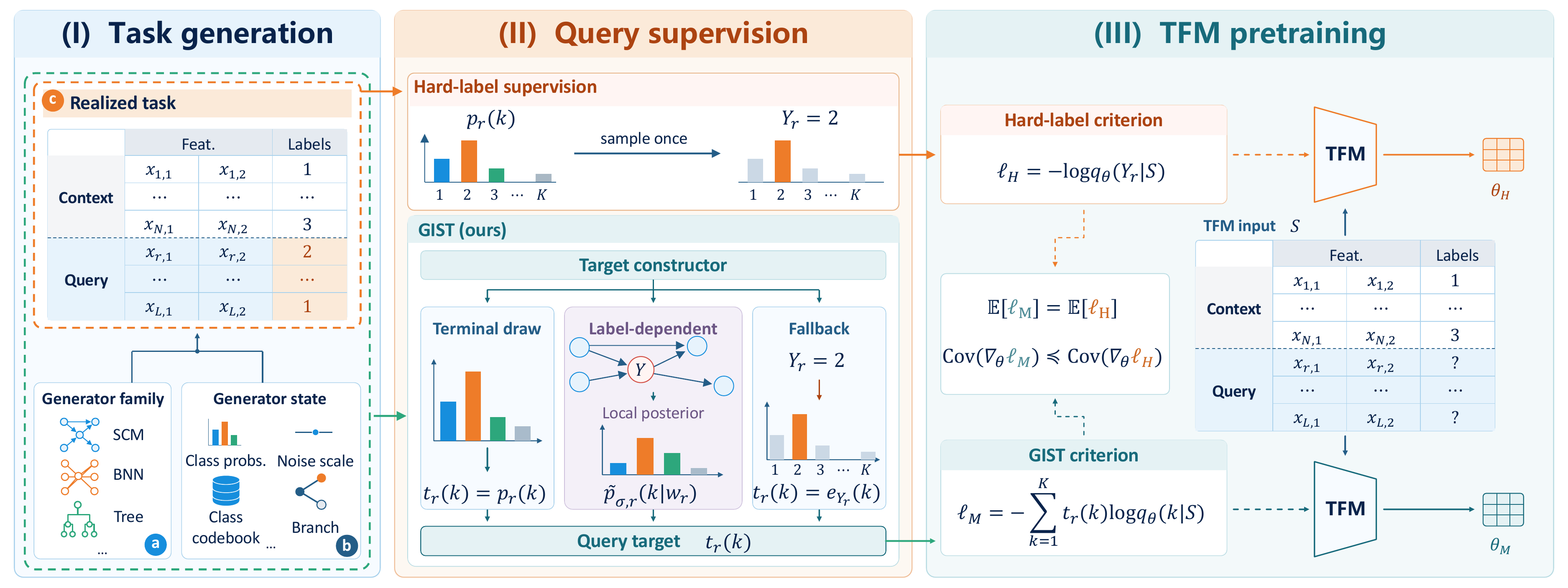}
    \caption{\textbf{Overview of \method{}.}
    (I) Task generation retains state and calibrates supported category encodings.
    (II) Query targets use conditional probabilities, with hard-label fallback.
    (III) The TFM uses standard inputs and hard context labels.
    Objective and raw-gradient covariance comparisons use the same generating distribution at fixed parameters.}
    \label{fig:overview}
    \vspace{-1.5em}
\end{figure}

Tabular prediction estimates a categorical or numerical target from a row of features, supporting applications such as credit assessment and clinical risk prediction.
Tabular foundation models (TFMs) such as TabPFN and TabICL predict targets for query rows through in-context learning, using labeled rows as context~\cite{hollmann2025tabpfn,qu2025tabicl}.
Many follow the prior-data fitted network (PFN) approach: a generator samples a data-generating mechanism from a prior, produces a synthetic table, and supplies context and query rows for pretraining~\cite{muller2022transformers}.

Synthetic priors determine the feature-label relationships encountered during pretraining.
Priors based on Bayesian neural networks (BNNs) and structural causal models (SCMs) provide synthetic tasks for TabPFN~\cite{hollmann2023tabpfn}, while mixtures of neural and tree-based SCMs supply tasks for TabICL~\cite{qu2025tabicl}.
Randomized graph-based SCMs in TabICLv2~\cite{qu2026tabiclv2} and mixtures of complementary priors in Mitra~\cite{zhang2025mitra} further diversify the generated tasks.

Standard PFN pretraining evaluates query predictions against sampled hard labels, yielding an unbiased Monte Carlo estimate of the population loss and its gradient~\cite{muller2022transformers,nagler2023statistical}.
Although synthetic tasks can be generated continually, the model processes only a finite number within its computational budget.
For a fixed input, a sampled label selects one class's loss even when the generator can provide conditional probabilities for all classes.
This raises our central research question: \emph{Can we use information retained during synthetic task generation to reduce the variability introduced by query-label sampling during pretraining?}

We propose \methodfull{} (\method{}), a framework that \textbf{couples synthetic task generation with conditional query supervision}.
Its statistical basis is Rao-Blackwellization (RB)~\cite{CASELLA1996RB}: replacing the sampled hard-label query loss with its conditional expectation given the model input and selected generator information.
This conditioning fixes the prediction, allowing \method{} to average over the remaining query-label uncertainty whenever the conditional probabilities are tractable.
Under a fixed generating distribution, we show that this averaging preserves the population objective and yields no larger gradient covariance at fixed model parameters.
The retained generator information is used only to construct targets; the predictor receives standard context and query inputs (Figure~\ref{fig:overview}).

The target construction depends on how the sampled label enters input generation.
For a query-label draw that does not affect the model input or task acceptance, the sampling probabilities give the target directly.
If the sampled class influences generated features, those features provide additional evidence about the label.
For supported label-dependent branches, \method{} represents the sampled class by a category code perturbed with Gaussian noise and routes all label-dependent feature computations through this encoding. Its local likelihood then yields the posterior target without evaluating a marginal likelihood through the downstream graph.
An entropy budget calibrates the noise scale by bounding average residual label uncertainty in this channel.
This encoding changes the affected branch's feature-label distribution; unsupported branches retain hard-label supervision.

To assess \emph{whether \method{} improves predictive performance across architectures}, we compare it with original-prior hard-label controls under matched pretraining schedules on TabICLv2~\cite{qu2026tabiclv2}, TabPFN-2.5~\cite{grinsztajn2025tabpfn25}, and TabSwift~\cite{liu2026tabswift}.
Evaluation covers IID classification on TabArena and TALENT~\cite{erickson2025tabarena,liu2025talenttoolbox}, and temporal and grouped OOD tasks on BeyondArena~\cite{purucker2026beyondarena}.
At the final matched checkpoints, \method{} improves mean rank and wins on a majority of datasets in all six IID architecture--benchmark combinations.
We separately examine \emph{where and when conditional supervision reduces clipped-gradient variance}, using fixed-checkpoint diagnostics under a shared calibrated prior across training stages, table lengths, and model components.
Finally, a TabSwift budget sweep assesses \emph{how sensitive predictive performance is to the entropy budget}.
Our contributions are as follows:
\begin{itemize}[leftmargin=1em]
    \item We \textbf{jointly design synthetic task generation and query supervision}, combining direct conditional targets with entropy-calibrated Gaussian encodings and their posterior targets.
    \item We establish \textbf{population-objective preservation and a gradient covariance reduction identity} for conditional supervision with hard-label fallback under a fixed generating distribution.
    \item Experiments on over 200 tabular datasets from TabArena and TALENT show that \method{} improves predictive performance across three representative TFM architectures, winning \textbf{62.1\%} of paired dataset-level comparisons against the corresponding hard-label controls.

\end{itemize}

% !TeX root = ../simrb.tex
\section{Preliminaries and Problem Setup}
\label{sec:preliminaries}

We study PFN-style tabular foundation models pretrained on synthetic classification tasks~\cite{hollmann2025tabpfn,qu2025tabicl}.
A generator samples a task mechanism \(T\sim\Pi\) and a context size \(N\sim\Pi_N\), then produces a labeled context \(D_N=\{(X_i,Y_i)\}_{i=1}^{N}\) and a query \((X,Y)\).
The task prior \(\Pi\) specifies a distribution over data-generating mechanisms, such as structural causal models and their parameters.
The predictor receives \(S=(D_N,X,N)\) and outputs class probabilities \(q_\theta(\cdot\mid S)\) over \(K\) classes.
We use one query for notation; with multiple queries, \(S\) includes the full visible input and the loss is averaged over query labels.
Unless stated otherwise, expectations and conditional distributions are taken under the pretraining distribution after preprocessing and task filtering.

The population objective is the expected hard-label cross-entropy~\cite{nagler2023statistical}.
For \(m\) independent synthetic episodes, its unbiased Monte Carlo estimate is
\begin{equation}
    R(\theta)=\mathbb{E}[-\log q_\theta(Y\mid S)],
    \qquad
    \widehat R_m(\theta)
    =-\frac{1}{m}\sum_{j=1}^{m}\log q_\theta(Y_j\mid S_j).
    \label{eq:pretraining-objective}
\end{equation}
Here \(m\) counts synthetic episodes, whereas \(N\) counts context observations.
Task, input, and query-label sampling contribute to gradient variance.

Let \(Z\) collect the full model input \(S\) and selected information retained during task generation, such as task parameters and intermediate encodings.
Define
\begin{equation}
    p_Z(k)=P(Y=k\mid Z),
    \qquad
    \pi_S(k)=P(Y=k\mid S)=\mathbb{E}[p_Z(k)\mid S].
    \label{eq:conditional-label-distributions}
\end{equation}
The posterior predictive distribution \(\pi_S\) minimizes \(R\) among all predictors that depend only on \(S\)~\cite{nagler2023statistical}.
The conditional query distribution \(p_Z\) supplies supervision when tractable, while the predictor receives only \(S\).
At a fixed input \(S\), averaging these targets over the additional generator information recovers \(\pi_S\).
Our goal is to construct synthetic tasks and retained states that make these conditional targets tractable on supported branches (Section~\ref{sec:method}).
Within each fixed generating distribution, we compare conditional averaging with hard-label supervision for the same population objective \(R(\theta)\).

% !TeX root = ../simrb.tex
\section{Method}
\label{sec:method}

\method{} constructs conditional query supervision from retained generator information.
We first define the supervision rule and identify when its targets are directly available, then construct and calibrate an encoding for supported label-dependent branches.
We conclude with the combined training objective and its statistical guarantees.

\subsection{Generator-informed query supervision}
\label{sec:structured-targets}

Let \(Z\) retain the complete model input \(S\) and selected generator information, and write \(p_Z(k)=P(Y=k\mid Z)\) (Section~\ref{sec:preliminaries}).
For a query with hard-label loss \(\ell_{\mathrm H}(\theta;S,Y)=-\log q_\theta(Y\mid S)\), \method{} uses the conditional loss whenever \(p_Z\) is tractable:
\begin{equation}
    \ell_{\mathrm{RB}}(\theta;Z)
    :=\mathbb{E}[\ell_{\mathrm H}(\theta;S,Y)\mid Z]
    =-\sum_{k=1}^{K}p_Z(k)\log q_\theta(k\mid S).
    \label{eq:rb-loss}
\end{equation}
Given \(Z\), the model prediction is fixed, so \(\ell_{\mathrm{RB}}\) averages the hard-label loss over the remaining query-label uncertainty.
Proposition~\ref{prop:rb} shows that this Rao-Blackwellization preserves the population objective and yields no larger gradient covariance under a fixed generating distribution and at fixed model parameters.
The construction of \(p_Z\) must account for the final input, so it depends on whether the sampled query class enters input generation.

During table synthesis, the generator samples a label for each query row.
We call this query-label draw \emph{terminal} when, given the retained pre-draw state, its outcome does not affect the model input.
Its sampling probabilities then give the conditional target directly, provided task filtering is also independent of the remaining query-label randomness under the same conditioning (Appendix~\ref{app:selection}).
This includes converters that pass a continuous state downstream independently of the sampled class, even when the node has observed descendants.
For MLP and tree generators with terminal target noise, we integrate that noise through class boundaries fitted on the context.

When the sampled class affects observed features (Figure~\ref{fig:channel-calibration}(I)), those features supply additional label evidence, so the pre-draw probabilities must be updated.
Direct evaluation of this posterior can require inference through the downstream graph.
For supported label-dependent branches, we jointly design the category encoding and its supervision rule so that the update can be computed from a local likelihood.
Branches without a supported conditional target retain sampled hard-label supervision.

\subsection{Gaussian encoding for label-dependent features}
\label{sec:gaussian-channel}

% !TeX root = ../simrb.tex
\begin{figure}[t]
    \centering
    \includegraphics[width=\linewidth]{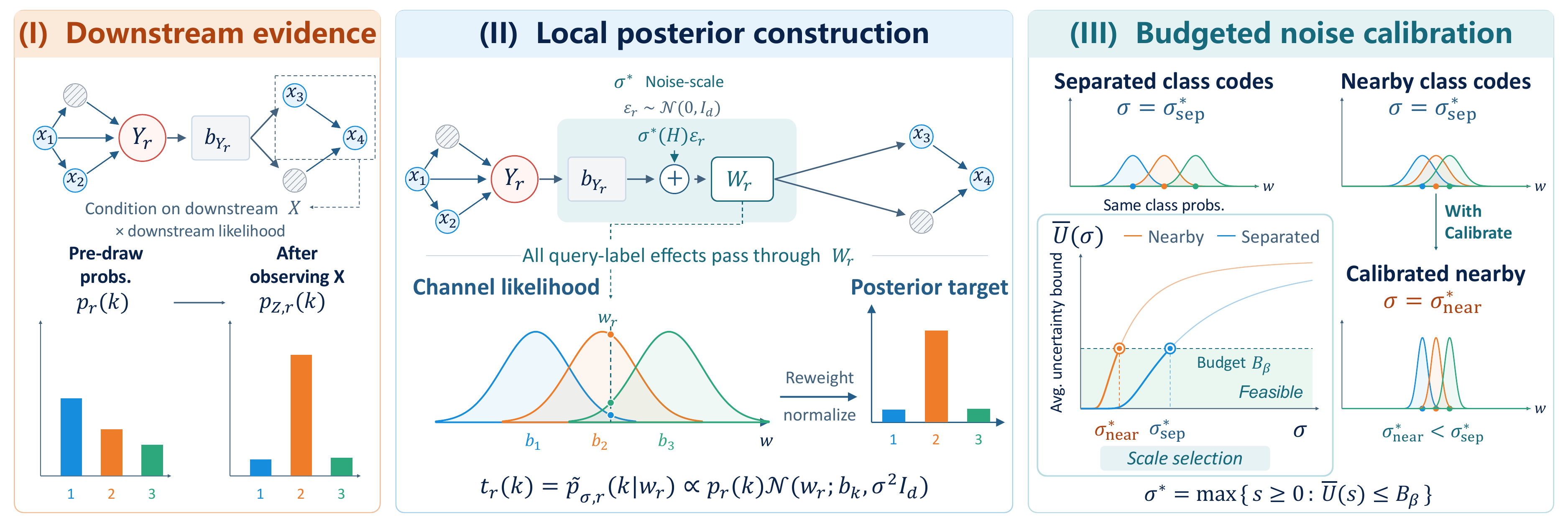}
    \caption{\textbf{Gaussian encoding and entropy-budget calibration.}
    (I) Observing label-dependent features updates class probabilities.
    (II) Routing all query-label effects through a local Gaussian channel makes posterior targets tractable.
    (III) Calibration selects the largest noise scale certified by an entropy budget; nearby codes require a smaller scale in the illustrated example.}
    \label{fig:channel-calibration}
    \vspace{-1.5em}
\end{figure}

In the supported label-dependent converters, a sampled class \(Y_r\) is mapped to a real-valued category code \(b_{Y_r}\) for downstream feature generation.
\method{} adds Gaussian noise to this code and routes all sampled-class effects on downstream features through the resulting encoding \(W_r\) (Figure~\ref{fig:channel-calibration}(II)).
Retaining this encoding allows the conditional target to be computed from its local likelihood.

\begin{samepage}
We describe the encoding channel before task filtering.
Let \(\mathcal H\) be the state before target-category draws, containing a codebook \(B=(b_1,\ldots,b_K)\), with \(b_k\in\mathbb R^d\), and row probabilities \(p_r\).
For \(L=N+Q\) rows, comprising \(N\) context rows and \(Q\) queries, we sample
\begin{equation}
    Y_r\mid\mathcal H\sim\operatorname{Categorical}(p_r),
    \qquad
    W_r=b_{Y_r}+\sigma\varepsilon_r,
    \qquad
    \varepsilon_r\sim\mathcal N(0,I_d).
    \label{eq:gaussian-channel}
\end{equation}
\end{samepage}
Conditional on \(\mathcal H\), the row labels are independent, and the encoding noises are independent across rows and independent of all labels.
The scale \(\sigma=\sigma(\mathcal H)>0\) depends only on pre-draw information and is shared across rows and classes.
The sampled class \(Y_r\) remains the dataset label, while downstream nodes receive the noisy code \(W_r\).
Section~\ref{sec:entropy-budget} specifies how the scale is chosen from \(\mathcal H\).

Retain \(Z_\sigma=(\mathcal H,W_{1:L},Y_{1:N},S)\), excluding query labels and the realized encoding noise \(\varepsilon_{1:L}\).
Conditional on \(\mathcal H\) and the context labels, all query-label effects on the input and task filtering pass through \(W_{1:L}\).
Formally, under the channel law before task filtering,
\[
    Y_{N+1:L}\perp(S,E)\mid(\mathcal H,W_{1:L},Y_{1:N}),
\]
where \(E\) denotes task acceptance.
Once these encodings are fixed, the downstream likelihood is common to all candidate query-label vectors and cancels in Bayes' rule.
The conditional target for query \(r>N\) is therefore
\begin{equation}
    \widetilde p_{\sigma,r}(k)
    =\frac{p_{rk}\exp\!\left[-\|W_r-b_k\|_2^2/(2\sigma^2)\right]}
    {\sum_{j=1}^{K}p_{rj}\exp\!\left[-\|W_r-b_j\|_2^2/(2\sigma^2)\right]}.
    \label{eq:channel-posterior}
\end{equation}
On accepted tasks, \(\widetilde p_{\sigma,r}(k)=P(Y_r=k\mid Z_\sigma)=p_{Z_\sigma,r}(k)\), supplying the conditional probabilities in Equation~\eqref{eq:rb-loss}.
The same channel generates the category representation and defines its supervision target.
This local encoding change modifies the branch's feature-label distribution.
Appendix~\ref{app:channel-posterior} proves the posterior identity after filtering, and Appendix~\ref{app:converter-mapping} gives the GraphSCM instantiations.

\subsection{Entropy-budget calibration}
\label{sec:entropy-budget}

The noise scale controls how much label uncertainty remains after the encoding is observed.
With distinct class codes, an almost deterministic encoding leaves little uncertainty to average, whereas excessive noise weakens the label information carried by the encoding.
A fixed scale may leave little uncertainty for well-separated codes but obscure label information for nearby codes (Figure~\ref{fig:channel-calibration}(III)).
We therefore select the largest scale certified by an entropy budget, adapting it to the task's class probabilities and codebook.

\begin{samepage}
Fix \(\mathcal H=h\) and, under the channel law before task filtering, define the local conditional entropy
\(\mathcal E_r(\sigma)=H(Y_r\mid W_r,\mathcal H=h)\), averaging over labels and encoding noise.
For a candidate noise scale, we evaluate an entropy upper bound from the row probabilities and pairwise code distances before labels and encoding noise are sampled.
The Bhattacharyya lower bound on mixture entropy~\cite{kolchinsky2017mixture} yields
\begin{equation}
    \begin{aligned}
        U_r(\sigma)
        &:=H(p_r)+\sum_i p_{ri}\log\!\left[
            \sum_j p_{rj}\exp\!\left(-\frac{\|b_i-b_j\|_2^2}{8\sigma^2}\right)
        \right],\\
        0&\le\mathcal E_r(\sigma)\le U_r(\sigma)\le H(p_r),
    \end{aligned}
    \label{eq:entropy-upper-bound}
\end{equation}
\end{samepage}
where \(H(p_r)=-\sum_k p_{rk}\log p_{rk}\) and logarithms are natural.

Define the all-row averages \(\overline U(s)=L^{-1}\sum_{r=1}^{L}U_r(s)\) and \(\overline H=L^{-1}\sum_{r=1}^{L}H(p_r)\).
For \(\beta\in(0,1]\), our experiments use the binary-cap budget \(\mathcal B_\beta(h)=\beta\min\{\overline H,\log 2\}\).
The cap limits the residual-entropy allowance to \(\beta\log 2\) when \(\overline H>\log 2\).
Subject to an optional cap \(\sigma_{\max}\in(0,\infty]\), we select the largest finite scale certified by this budget.
Uncapped calibration uses \(0<\beta<1\); when \(\beta=1\), we require a finite \(\sigma_{\max}\).
When a positive finite maximizer exists,
\begin{equation}
    \sigma_\star(\mathcal H)
    =\max\left\{
        s\ge0:\ s\le\sigma_{\max},\quad
        \overline U(s)\le\mathcal B_\beta(\mathcal H)
    \right\},
    \label{eq:entropy-scale}
\end{equation}
using the right-hand limit at \(s=0\).

At a feasible positive scale, Equations~\eqref{eq:entropy-upper-bound} and~\eqref{eq:entropy-scale} give \(L^{-1}\sum_{r=1}^{L}\mathcal E_r(\sigma_\star)\le\overline U(\sigma_\star)\le\mathcal B_\beta(h)\).
Equivalently, the encoding retains at least \(\overline H-\mathcal B_\beta(h)\ge(1-\beta)\overline H\) of label information on average across all rows before task filtering.

Since \(\overline U(s)\) is continuous and nondecreasing for \(s>0\), calibration selects a feasible finite cap directly or locates the budget threshold by bisection.
Infeasible budgets retain the original encoding and sampled hard labels.
Other cases in which no positive scale is returned follow the branch rules in Appendix~\ref{app:entropy-calibration}.
Appendices~\ref{app:entropy-bound}--\ref{app:budget-policies} give the derivations, numerical procedure, and alternative budget policy.

\subsection{Training objective and statistical guarantees}
\label{sec:conditional-supervision}

The preceding constructions supply a target \(t_r=p_{Z,r}\) on supported branches, using either terminal probabilities or the channel posterior, and \(t_r=e_{Y_r}\) on hard-label fallback branches.
Here \(p_{Z,r}(k)=P(Y_r=k\mid Z)\) uses a common retained state \(Z\) for the table, and \(e_{Y_r}\) is the sampled label's one-hot vector.
For \(Q\) queries with predictions \(q_{\theta,r}\), \method{} uses
\begin{equation}
    \ell_{\mathrm M}^{\mathrm{table}}(\theta)
    =-\frac1Q\sum_{r=N+1}^{L}\sum_{k=1}^{K}
        t_r(k)\log q_{\theta,r}(k\mid S).
    \label{eq:gist-objective}
\end{equation}
Targets are held fixed during differentiation.
The loss reuses the model prediction; context labels remain hard, and the architecture and inference procedure are unchanged.

For the single-query losses in Section~\ref{sec:structured-targets}, write \(G_a=\nabla_\theta\ell_a\) and \(V_a=\operatorname{Cov}(G_a)\).
Let \(P\) denote the joint law of \((S,Y,Z)\) induced by a specified generator, calibration rule, and task filter, independent of \(\theta\).
The comparisons below use \(P\) at fixed parameters and model randomness.
We assume integrable losses, finite gradient second moments, and interchangeability of differentiation and conditional expectation (Appendix~\ref{app:rb-assumptions}).

\begin{proposition}[Conditional supervision]
\label{prop:rb}
If \(p_Z(k)=P(Y=k\mid Z)\), then
\begin{equation}
    \mathbb{E}[\ell_{\mathrm{RB}}(\theta;Z)]=R(\theta),
    \quad G_{\mathrm{RB}}=\mathbb{E}[G_{\mathrm H}\mid Z],
    \quad V_{\mathrm H}-V_{\mathrm{RB}}
    =\mathbb{E}\!\left[\operatorname{Cov}(G_{\mathrm H}\mid Z)\right]\succeq0.
    \label{eq:rb-guarantee}
\end{equation}
\end{proposition}
\begin{proof}
Given \(Z\), the model prediction is fixed, so averaging the hard-label loss gives Equation~\eqref{eq:rb-loss}.
Taking expectations preserves \(R\), and differentiating gives the conditional-gradient identity.
The covariance equality follows from the law of total covariance.
\end{proof}

For selective supervision, let \(A=A(Z)\in\{0,1\}\) select tractable branches before the residual label draw, and write \(\Gamma_\theta(Z)=\operatorname{Cov}(G_{\mathrm H}\mid Z)\).
The single-query form of Equation~\eqref{eq:gist-objective} then satisfies
\begin{equation}
    \ell_{\mathrm M}=A\ell_{\mathrm{RB}}+(1-A)\ell_{\mathrm H},
    \quad \mathbb{E}[\ell_{\mathrm M}]=R(\theta),
    \quad V_{\mathrm H}-V_{\mathrm M}
    =\mathbb{E}[A\Gamma_\theta(Z)]\succeq0.
    \label{eq:selective-rb}
\end{equation}
The mixed objective retains the posterior predictive target \(\pi_S\) and covariance ordering under \(P\) (Appendix~\ref{app:calibrated-objective}).
The table-level extension uses a common retained state and includes cross-query covariance terms (Appendix~\ref{app:multiple-queries}).

For softmax predictions \(q_\theta=\operatorname{softmax}(a_\theta)\), let \(J_\theta=\partial a_\theta/\partial\theta^\top\) be the logit Jacobian.
The removable covariance for a single query is
\begin{equation}
    \Gamma_\theta(Z)
    =J_\theta^\top
    \left[\operatorname{Diag}(p_Z)-p_Zp_Z^\top\right]J_\theta.
    \label{eq:removable-noise}
\end{equation}
The bracketed term is the conditional covariance of the one-hot label, mapped to parameter space by the logit Jacobian.
Deterministic conditional labels give \(\Gamma_\theta(Z)=0\), while task and input variation remain after averaging.
Appendix~\ref{app:rb-proofs} gives the covariance decomposition and gradient estimation error, along with an asymptotic comparison under additional M-estimation assumptions.
Proposition~\ref{prop:rb} concerns raw gradients; Section~\ref{sec:component-analysis} evaluates clipped-gradient variability empirically.

% !TeX root = ../simrb.tex
\section{Experiments}
\label{sec:experiments}
% Generated by scripts/build_experiment_assets.py.
\newcommand{\TalentMainDatasets}{188}
\newcommand{\TalentMainFolds}{940}
\newcommand{\ArenaICLWin}{55.3}
\newcommand{\ArenaICLWins}{21}
\newcommand{\ArenaICLHardImp}{9.04}
\newcommand{\ArenaICLHardRank}{5.62}
\newcommand{\ArenaICLRBImp}{9.23}
\newcommand{\ArenaICLRBRank}{5.37}
\newcommand{\ArenaSwiftWin}{65.8}
\newcommand{\ArenaSwiftWins}{25}
\newcommand{\ArenaSwiftHardImp}{15.75}
\newcommand{\ArenaSwiftHardRank}{8.11}
\newcommand{\ArenaSwiftRBImp}{15.31}
\newcommand{\ArenaSwiftRBRank}{7.70}
\newcommand{\ArenaPFNWin}{71.1}
\newcommand{\ArenaPFNWins}{27}
\newcommand{\ArenaPFNHardImp}{8.61}
\newcommand{\ArenaPFNHardRank}{4.42}
\newcommand{\ArenaPFNRBImp}{4.57}
\newcommand{\ArenaPFNRBRank}{3.74}
\newcommand{\TalentICLWin}{64.4}
\newcommand{\TalentICLWins}{121}
\newcommand{\TalentICLHardImp}{12.66}
\newcommand{\TalentICLHardRank}{4.56}
\newcommand{\TalentICLRBImp}{11.62}
\newcommand{\TalentICLRBRank}{4.09}
\newcommand{\TalentSwiftWin}{67.0}
\newcommand{\TalentSwiftWins}{126}
\newcommand{\TalentSwiftHardImp}{20.99}
\newcommand{\TalentSwiftHardRank}{6.66}
\newcommand{\TalentSwiftRBImp}{19.91}
\newcommand{\TalentSwiftRBRank}{6.24}
\newcommand{\TalentPFNWin}{53.7}
\newcommand{\TalentPFNWins}{101}
\newcommand{\TalentPFNHardImp}{7.52}
\newcommand{\TalentPFNHardRank}{3.32}
\newcommand{\TalentPFNRBImp}{7.77}
\newcommand{\TalentPFNRBRank}{3.23}
\newcommand{\BeyondICLWin}{50.0}
\newcommand{\BeyondICLWins}{6}
\newcommand{\BeyondICLHardImp}{27.45}
\newcommand{\BeyondICLHardRank}{7.21}
\newcommand{\BeyondICLRBImp}{27.33}
\newcommand{\BeyondICLRBRank}{7.26}
\newcommand{\BeyondSwiftWin}{75.0}
\newcommand{\BeyondSwiftWins}{9}
\newcommand{\BeyondSwiftHardImp}{27.12}
\newcommand{\BeyondSwiftHardRank}{7.81}
\newcommand{\BeyondSwiftRBImp}{25.05}
\newcommand{\BeyondSwiftRBRank}{7.11}
\newcommand{\BeyondPFNWin}{58.3}
\newcommand{\BeyondPFNWins}{7}
\newcommand{\BeyondPFNHardImp}{18.31}
\newcommand{\BeyondPFNHardRank}{4.19}
\newcommand{\BeyondPFNRBImp}{17.95}
\newcommand{\BeyondPFNRBRank}{4.68}

% Generated by scripts/build_experiment_assets.py.
\newcommand{\BeyondOwnICLWins}{11}
\newcommand{\BeyondOwnICLDatasets}{21}
\newcommand{\BeyondOwnSwiftWins}{15}
\newcommand{\BeyondOwnSwiftDatasets}{18}
\newcommand{\BeyondOwnPFNWins}{7}
\newcommand{\BeyondOwnPFNDatasets}{12}

We evaluate GIST as a joint design of synthetic task generation and query supervision. Our experiments address three complementary questions:
(i) does the complete design improve downstream prediction across TFM architectures under matched pretraining schedules, in both IID and OOD evaluation?
(ii) under a shared calibrated prior, where and when does conditional supervision reduce clipped-gradient variability?
(iii) how sensitive is downstream performance to the entropy budget controlling the encoding channel?

\subsection{Experimental setup}
\label{sec:experiment-setup}

\stitle{Datasets.}
IID evaluation covers 38 TabArena classification datasets and \TalentMainDatasets{} TALENT datasets shared by all methods~\cite{erickson2025tabarena,liu2025talenttoolbox}.
OOD evaluation uses BeyondArena's temporal and grouped splits~\cite{purucker2026beyondarena}, with each architecture's complete feasible task set: \BeyondOwnICLDatasets{} datasets for TabICLv2, \BeyondOwnSwiftDatasets{} for TabSwift, and \BeyondOwnPFNDatasets{} for TabPFN-2.5.
IID and OOD refer to the real-data train/test split regimes.

\stitle{Baselines.}
We evaluate Hard and \method{} on three representative TFM architectures: TabICLv2 with column-then-row encoding~\cite{qu2026tabiclv2}, TabSwift with efficient row-wise attention~\cite{liu2026tabswift}, and TabPFN-2.5 with cell-based alternating row/column attention~\cite{grinsztajn2025tabpfn25}.
Within TabICLv2's GraphSCM generator family, Hard uses the original generator and hard query labels; \method{} uses our joint generator--supervision design.
We also include tuned CatBoost, LightGBM, XGBoost, RealMLP, and TabM baselines~\cite{prokhorenkova2018catboost,ke2017lightgbm,chen2016xgboost,holzmuller2024realmlp,gorishniy2025tabm}.

\stitle{Metrics.}
We report ROC-AUC for binary tasks and log loss for multiclass tasks, averaging within datasets before weighting datasets equally.
For cross-task aggregation, error is \(e=1-\mathrm{AUC}\) or log loss, respectively.
On each split, we compute Improvability~\cite{erickson2025tabarena} as \(I_m=1-\min_j e_j/e_m\), with \(I_m=0\) when \(e_m=0\); mean rank uses average ranks for ties.
Both are lower-is-better and follow the same dataset weighting.
Paired win rates compare dataset-level mean scores within each architecture, retaining ties in the denominator.

\stitle{Training and evaluation protocols.}
Within each architecture, Hard and \method{} use matched effective global batch sizes, training schedules, and inference settings.
Stage~1 (short-context pretraining) trains for 30,000 updates on synthetic tasks of fixed length 1,024 rows.
Stage~2 (long-context continuation) trains for 6,000 additional updates from the corresponding Stage~1 checkpoint, with task lengths sampled logarithmically between 400 and 10,240 rows.
Task length counts both support and query rows.
Unless stated otherwise, Hard and \method{} use the final checkpoints of this two-stage, 36K-update schedule.
By default, \method{} uses \(\beta=0.15\), the all-row binary-cap budget, and no external noise cap.
TabArena evaluation uses 594 official outer splits with the benchmark's internal validation and refitting procedures; TALENT uses five stratified folds per dataset.
BeyondArena follows the official outer splits without inner cross-validation for TFMs.
Appendix~\ref{app:experiment-details} specifies tuning, inference, and dataset coverage.

\stitle{Environment.}
Our TFM runs use PPU-ZW810E servers with 16 accelerators per node (95.6\,GiB each), running Ubuntu 24.04.2 LTS, PPU SDK 2.1.0, and driver 1.3.2.
Pretraining uses Python 3.12 and a PPU build of PyTorch 2.9.0, with FP32 and single-node distributed data parallelism.

\subsection{Predictive performance}
\label{sec:predictive-performance}

We evaluate TabICLv2, TabSwift, and TabPFN-2.5 trained with \method{} against their Hard counterparts, which use the original GraphSCM generator and sampled hard labels.
We also include five representative baselines: three gradient-boosted tree methods (CatBoost, LightGBM, and XGBoost) and two neural models (RealMLP and TabM).
Evaluation covers the IID splits of TabArena and TALENT and the OOD splits of BeyondArena.
Table~\ref{tab:benchmark-results} summarizes IID performance across all methods, while Table~\ref{tab:paired-win-rates} reports \method{}'s paired win rates against Hard within each architecture across the three benchmarks.

% Generated by scripts/build_experiment_assets.py.
\begin{table}[!t]
\centering
\caption{\textbf{IID performance of the complete \method{} design.}
AUC: binary ROC-AUC; LL: multiclass log loss; $I$: Improvability (\%); Rank: mean rank among 11 methods.
Bold marks the best value per column; \textcolor{gistImproved}{blue}/\textcolor{gistNotImproved}{green} marks improvement/deterioration against same-architecture Hard.}
\label{tab:benchmark-results}
\begingroup
\fontsize{9}{10.5}\selectfont
\setlength{\tabcolsep}{2.2pt}
\renewcommand{\arraystretch}{1.10}
\begin{tabular*}{\linewidth}{@{\extracolsep{\fill}}lrrrrrrrr@{}}
\toprule
\multirow{2}{*}[-2pt]{Method} & \multicolumn{4}{c}{TabArena (38)} & \multicolumn{4}{c}{TALENT (188)} \\
\cmidrule(lr){2-5}\cmidrule(l){6-9}
& AUC$\uparrow$ (30) & LL$\downarrow$ (8) & $I\downarrow$ & Rank$\downarrow$ & AUC$\uparrow$ (120) & LL$\downarrow$ (68) & $I\downarrow$ & Rank$\downarrow$ \\
\midrule
CatBoost & 0.8496 & 0.2611 & 11.47 & 6.03 & 0.8536 & 0.4273 & 26.08 & 7.09 \\
LightGBM & 0.8480 & 0.2614 & 12.52 & 6.20 & 0.8606 & 0.4503 & 26.57 & 7.36 \\
XGBoost & 0.8473 & 0.2610 & 13.08 & 6.89 & 0.8481 & 0.4735 & 28.79 & 7.61 \\
RealMLP & 0.8491 & 0.2543 & 10.86 & 5.98 & 0.8593 & 0.4572 & 29.73 & 8.14 \\
TabM & 0.8488 & 0.2615 & 11.43 & 5.93 & 0.8513 & 0.4422 & 28.58 & 7.70 \\
\midrule
TabICLv2 (Hard) & 0.8495 & 0.2541 & 9.04 & 5.62 & 0.8804 & 0.3977 & 12.66 & 4.56 \\
\quad + \method{} & \textcolor{gistImproved}{0.8501} & \textcolor{gistNotImproved}{0.2549} & \textcolor{gistNotImproved}{9.23} & \textcolor{gistImproved}{5.37} & \textcolor{gistImproved}{0.8817} & \textcolor{gistImproved}{0.3967} & \textcolor{gistImproved}{11.62} & \textcolor{gistImproved}{4.09} \\
\addlinespace[2pt]
TabSwift (Hard) & 0.8420 & 0.2848 & 15.75 & 8.11 & 0.8713 & 0.4267 & 20.99 & 6.66 \\
\quad + \method{} & \textcolor{gistImproved}{0.8425} & \textcolor{gistImproved}{0.2811} & \textcolor{gistImproved}{15.31} & \textcolor{gistImproved}{7.70} & \textcolor{gistImproved}{0.8726} & \textcolor{gistImproved}{0.4243} & \textcolor{gistImproved}{19.91} & \textcolor{gistImproved}{6.24} \\
\addlinespace[2pt]
TabPFN-2.5 (Hard) & 0.8495 & 0.2884 & 8.61 & 4.42 & 0.8863 & \textbf{0.3909} & \textbf{7.52} & 3.32 \\
\quad + \method{} & \textcolor{gistImproved}{\textbf{0.8529}} & \textcolor{gistImproved}{\textbf{0.2536}} & \textcolor{gistImproved}{\textbf{4.57}} & \textcolor{gistImproved}{\textbf{3.74}} & \textcolor{gistImproved}{\textbf{0.8864}} & \textcolor{gistNotImproved}{0.3915} & \textcolor{gistNotImproved}{7.77} & \textcolor{gistImproved}{\textbf{3.23}} \\
\bottomrule
\end{tabular*}
\endgroup
\end{table}

\stitle{IID prediction.}
Across the six combinations of architecture and benchmark, \method{} improves mean rank and mean binary ROC-AUC in all six, and mean multiclass log loss and Improvability in four (Table~\ref{tab:benchmark-results}).
It also wins on a majority of datasets in every combination (Table~\ref{tab:paired-win-rates}).
TabSwift improves all four reported metrics on both IID benchmarks, with paired win rates of \ArenaSwiftWin\% on TabArena and \TalentSwiftWin\% on TALENT.
TabPFN-2.5 achieves the largest decrease in Improvability on TabArena, from \ArenaPFNHardImp\% to \ArenaPFNRBImp\%, with \ArenaPFNWins/38 paired wins.

\stitle{OOD prediction.}
On the architecture-specific BeyondArena task sets, \method{} wins on \BeyondOwnSwiftWins/\BeyondOwnSwiftDatasets{} datasets for TabSwift, \BeyondOwnICLWins/\BeyondOwnICLDatasets{} for TabICLv2, and \BeyondOwnPFNWins/\BeyondOwnPFNDatasets{} for TabPFN-2.5 (Table~\ref{tab:paired-win-rates}).
Each comparison includes every task with all official splits completed for both conditions.
On the 12 datasets shared by all methods, \method{} lowers mean multiclass log loss for all three architectures; mean binary AUC improves for TabICLv2 and TabSwift but decreases for TabPFN-2.5.
Appendix~\ref{app:ood-results} gives the full metrics and comparisons with external baselines.
The results show that \method{} can improve prediction under the evaluated distribution shifts, with gains varying across architectures and metrics.

\subsection{Gradient variance across training and components}
\label{sec:component-analysis}

At seven checkpoints along each \method{} training trajectory, we compare Hard and selective \method{} under the same calibrated GraphSCM prior, sharing weights, inputs, logits, and loss normalization.
We use tasks with tractable conditional query targets at table lengths $L\in\{1024,8192\}$ with 80\% context rows and clip each whole-batch gradient at global norm one.
Figure~\ref{fig:gradient-stages} decomposes the percentage reduction $R$ in gradient covariance trace relative to Hard into component contributions $c_m$ in percentage points.
Within-component relative reductions use each component's own Hard variance as the baseline.
Pointwise 95\% bootstrap intervals quantify batch-sampling variability at fixed weights.
Appendices~\ref{app:gradient-diagnostics} and~\ref{app:module-analysis} give the estimator, decomposition, and parameter partitions.

% !TeX root = ../simrb.tex
\begin{figure}[t]
    \centering
    \includegraphics[width=\linewidth]{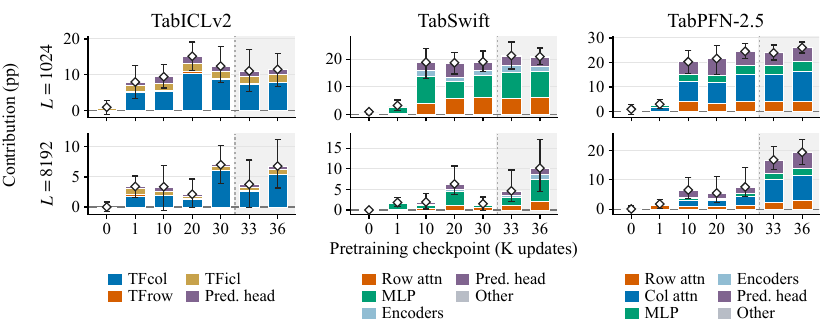}
    \caption{\textbf{Component contributions to clipped-gradient variance reduction.}
    Segments show contributions $c_m$ in percentage points of whole-model Hard variance; hollow diamonds show their sum $R$ with pointwise 95\% batch-pair bootstrap intervals.
    Shading marks Stage~2; checkpoints are equally spaced for display, and panels use separate linear $y$-axis ranges.}
    \label{fig:gradient-stages}
    \vspace{-1em}
\end{figure}

% !TeX root = ../../simrb.tex
\begin{table}[!t]
    \begin{minipage}[t]{0.485\linewidth}
        \vspace{0pt}
        \caption{\textbf{Paired dataset-level win rates (\%).}
        \method{} versus same-architecture Hard, retaining ties in the denominator.
        IID sets match Table~\ref{tab:benchmark-results}; BeyondArena uses architecture-specific complete sets (counts in parentheses).}
        \label{tab:paired-win-rates}
    \end{minipage}\hfill
    \begin{minipage}[t]{0.485\linewidth}
        \vspace{0pt}
        \caption{\textbf{TabSwift-\method{} entropy-budget sensitivity.}
        Each $\beta$ is fixed across both pretraining stages.
        Datasets and metrics follow Table~\ref{tab:benchmark-results}; bold marks the best mean per column.}
        \label{tab:beta-sensitivity}
    \end{minipage}
    \par\noindent
    \begin{minipage}[t]{0.485\linewidth}
        \vspace{0pt}
        \centering
        % Generated by scripts/build_experiment_assets.py.
\begingroup
\fontsize{9}{10.5}\selectfont
\setlength{\tabcolsep}{2pt}
\renewcommand{\arraystretch}{1.15}
\begin{tabular*}{\linewidth}{@{\extracolsep{\fill}}lrrr@{}}
\toprule
Architecture & TabArena & TALENT & BeyondArena \\
\midrule
TabICLv2 & 55.3 & 64.4 & 52.4\,(21) \\
TabSwift & 65.8 & 67.0 & 83.3\,(18) \\
TabPFN-2.5 & 71.1 & 53.7 & 58.3\,(12) \\
\bottomrule
\end{tabular*}
\endgroup

    \end{minipage}\hfill
    \begin{minipage}[t]{0.485\linewidth}
        \vspace{0pt}
        \centering
        % Generated by scripts/build_experiment_assets.py.
\begingroup
\fontsize{8.5}{9}\selectfont
\setlength{\tabcolsep}{2pt}
\renewcommand{\arraystretch}{0.92}
\setlength{\aboverulesep}{0.9pt}
\setlength{\belowrulesep}{1.25pt}
\begin{tabular*}{\linewidth}{@{\extracolsep{\fill}}lrrrr@{}}
\toprule
\multirow{2}{*}[-1pt]{$\beta$} & \multicolumn{2}{c}{TabArena} & \multicolumn{2}{c}{TALENT (188)} \\
\cmidrule(lr){2-3}\cmidrule(l){4-5}
& AUC$\uparrow$ & LL$\downarrow$ & AUC$\uparrow$ & LL$\downarrow$ \\
\midrule
Hard & 0.8420 & 0.2848 & 0.8713 & 0.4267 \\
0.05 & \textbf{0.8427} & 0.2824 & \textbf{0.8731} & 0.4250 \\
0.15 & 0.8425 & \textbf{0.2811} & 0.8726 & 0.4243 \\
0.30 & 0.8423 & 0.2822 & 0.8727 & \textbf{0.4240} \\
\bottomrule
\end{tabular*}
\endgroup

    \end{minipage}
    \vspace{-1.5em}
\end{table}

\stitle{Reduction beyond early training.}
The net reductions in Figure~\ref{fig:gradient-stages} are below 1\% at initialization for all three architectures and both table lengths.
At $L=1024$, TabSwift and TabPFN-2.5 already show most of their final reduction by 10K updates; at $L=8192$, their largest estimates occur at 36K.
At 36K updates, the reductions range from 6.80\% to 26.04\%, with all six pointwise intervals above zero.
The final reduction is smaller at $L=8192$ than at $L=1024$ for every architecture.
The larger effect at trained checkpoints is consistent with the parameter dependence in Equation~\eqref{eq:removable-noise}: before clipping, the logit Jacobian maps conditional label covariance into parameter space.
Even with fixed conditional targets, changes in this Jacobian can alter the covariance removed by averaging.

\stitle{Beyond the prediction head.}
At 36K updates, the prediction head has the largest estimated within-component relative reduction in all six settings, ranging from 37.70\% to 47.23\% (Appendix Tables~\ref{tab:variance-components-1024} and~\ref{tab:variance-components-8192}).
The head accounts for only 1.2\%--13.9\% of Hard gradient variance in these settings, so its large proportional response has limited weight in the whole-model total.
Components outside the head account for 72.0\%--91.5\% of the net whole-model reduction.
Although \method{} changes supervision at the query loss, its measured variance effect extends to the encoders and backbone components.

\stitle{Contributions across architectures.}
Column-processing components are major contributors across training in TabICLv2 and TabPFN-2.5.
TFcol has the largest estimated contribution at every measured noninitial checkpoint for both table lengths, accounting for more than half of the net reduction in each setting.
Within-column attention in TabPFN-2.5 ranks among the top two contributors at these checkpoints, leading at every checkpoint for $L=1024$ and from 20K onward for $L=8192$.
Both components aggregate information across samples within feature columns or groups.
In TabSwift, MLPs contribute the most in 11 of the 12 settings after initialization, including both final checkpoints.
The leading contributors are thus largely stable across trained checkpoints, but their identity differs across architectures.

% The full-mixture reductions are smaller, approximately 0\%--7.33\%; Appendix~\ref{app:gradient-diagnostics} reports these results and query-count controls.

\FloatBarrier
\subsection{Entropy-budget sensitivity}
\label{sec:beta-sensitivity}

The entropy budget controls the residual label uncertainty permitted in the encoding channel, making it a design parameter of the generator--supervision pair.
We pretrain TabSwift with \(\beta\in\{0.05,0.15,0.30\}\) under the same training schedule and evaluation protocol.
All three tested budgets improve mean binary AUC and multiclass log loss over Hard on both IID benchmarks (Table~\ref{tab:beta-sensitivity}).
The preferred value varies by benchmark and metric, but gains extend beyond the default budget of \(\beta=0.15\) to all three tested values.
Appendix~\ref{app:beta-sensitivity} gives paired win rates for each budget and the full protocol.

% !TeX root = ../simrb.tex
\section{Related Work}
\label{sec:related-work}

\paragraph{Tabular Prediction Tasks.}
Tabular prediction comprises classification and regression from structured records containing numerical and categorical features.
A table can mix continuous measurements, categorical identifiers, and text fields, and column meanings can differ across datasets.
This heterogeneity affects how models represent features and share information across tasks.
Most approaches fit a separate predictor to each dataset.
Gradient-boosted trees, including XGBoost, LightGBM, and CatBoost, are widely used in this setting~\cite{chen2016xgboost,ke2017lightgbm,prokhorenkova2018catboost}.
Neural methods learn feature representations and interactions, with transformer and retrieval-based models providing alternatives~\cite{gorishniy2021ftt,gorishniy2024tabr}.
AutoML systems combine model selection and ensembling to adapt these learners to individual datasets~\cite{erickson2020autogluon}.
OpenML-CC18, TabZilla, TALENT, and TabArena support comparisons across datasets with different sample sizes, feature types, and target distributions~\cite{bischl2021openml,mcelfresh2023tabzilla,ye2024closerlook,liu2025talenttoolbox,erickson2025tabarena}.

\paragraph{Tabular Foundation Models.}
TFMs learn across tables during pretraining and use that experience to predict on new datasets.
PFNs implement this idea through synthetic tasks sampled from a prior, whose posterior predictive distribution defines the population target~\cite{muller2022transformers,nagler2023statistical}.
The generator defines the task distribution, while the supervision rule determines how sampled tasks contribute to pretraining.
TabPFN, TabICL, and Mitra develop this approach through synthetic-prior construction and model design~\cite{hollmann2023tabpfn,hollmann2025tabpfn,qu2025tabicl,zhang2025mitra}.
TabICLv2, TabFlex, and TabSwift further address context scaling and architectural efficiency~\cite{qu2026tabiclv2,zeng2025tabflex,liu2026tabswift}.
TabDPT pretrains on real tables~\cite{ma2025tabdpt}, while Real-TabPFN and Real-TabPFN-2.5 continue pretraining synthetic-data models on real datasets~\cite{garg2025realtabpfn,grinsztajn2025tabpfn25}.
Semantic models exploit column names and cell text as transferable information~\cite{kim2024carte,spinaci2025contexttab,arazi2025tabstar}.

\paragraph{Generator-Informed Learning.}
PIQL augments TFMs with dataset statistics and generator-program representations, using a student module to reconstruct the latter from visible context at inference~\cite{ding2026privileged}.
\method{} uses generator information to construct query targets while retaining the predictor and inference procedure.
ALICE replaces sampled labels with conditional class probabilities derived from joint likelihood ratios for simulator-based likelihood-ratio estimation~\cite{stoye2018likelihood}.
\method{} shares this conditional-averaging principle and jointly designs task generation and supervision to obtain tractable posterior targets through an entropy-calibrated local channel in supported label-dependent branches.

% !TeX root = ../simrb.tex
\section{Conclusion}
\label{sec:conclusion}

Synthetic generators can support pretraining through the conditional label information they retain.
\method{} makes this information usable by jointly designing task generation and query supervision, enabling conditional loss averaging that preserves the population objective within each fixed generating distribution.
Improvements across three TFM architectures on TabArena and TALENT support the practical value of this design without changes to model architecture or inference.
These results motivate considering the supervision a synthetic generator can provide alongside the tasks it produces, with transfer under distribution shift remaining an open direction.

\subsection*{AI use statement}

We used generative AI tools to assist with drafting and revising selected parts of the manuscript, including figure and table captions, experimental discussion, and language polishing. Generative AI was also used to provide feedback on the organization of technical material, experimental analysis, and interpretation of reported results. The authors reviewed and revised all AI-assisted content and take full responsibility for the final manuscript.
\bibliographystyle{iclr2027_conference}
\bibliography{predtfm_bib}

\clearpage
\appendix
% !TeX root = ../simrb.tex
% !TeX root = ../simrb.tex
\section{Conditional Supervision: Proofs and Extensions}
\label{app:rb-proofs}

\subsection{Regularity conditions and objective preservation}
\label{app:rb-assumptions}

Use the law \(P\) from Section~\ref{sec:conditional-supervision}, with \(S\) measurable with respect to \(Z\).
Assume \(q_\theta(k\mid S)>0\) for classes with positive conditional probability.
The losses are integrable, their gradients have finite second moments, and differentiation can be interchanged with the expectations used below.
One sufficient local condition for the last requirement is differentiability together with an integrable bound on the gradient in a neighborhood of \(\theta\).
Model randomness is held fixed; alternatively, independent model randomness can be included in both \(S\) and \(Z\).

Since \(q_\theta(\cdot\mid S)\) is fixed conditional on \(Z\),
\begin{equation}
    \begin{aligned}
        \mathbb{E}[\ell_{\mathrm{H}}\mid Z]
        &=-\sum_{k=1}^{K}P(Y=k\mid Z)\log q_\theta(k\mid S)
        =\ell_{\mathrm{RB}},\\
        \mathbb{E}[\ell_{\mathrm{RB}}]
        &=\mathbb{E}\!\left[\mathbb{E}[\ell_{\mathrm{H}}\mid Z]\right]
        =R(\theta).
    \end{aligned}
    \label{eq:app-risk-preservation}
\end{equation}
To connect this equality to the model's prediction target, define
\(\ell_{\mathrm{P}}(\theta;S)=-\sum_k\pi_S(k)\log q_\theta(k\mid S)\).
The tower property gives
\begin{equation}
    \mathbb{E}[\ell_{\mathrm{RB}}\mid S]=\ell_{\mathrm{P}},
    \qquad
    \mathbb{E}[\ell_{\mathrm{P}}]=R(\theta).
    \label{eq:app-ppd-objective}
\end{equation}
Thus conditioning the supervision on generator information preserves the objective associated with the posterior predictive distribution.
The generator supplies \(p_Z\) during training, while the predictor continues to estimate \(\pi_S\) from \(S\).

\subsection{Gradient covariance and the softmax expression}
\label{app:softmax-noise}

Differentiating the conditional-loss identity yields \(G_{\mathrm{RB}}=\mathbb{E}[G_{\mathrm{H}}\mid Z]\).
Let \(\Delta=G_{\mathrm{H}}-G_{\mathrm{RB}}\).
Then \(\mathbb{E}[\Delta\mid Z]=0\), so its cross-covariance with any square-integrable function of \(Z\), including \(G_{\mathrm{RB}}\), vanishes.
Consequently,
\begin{equation}
    V_{\mathrm{H}}-V_{\mathrm{RB}}
    =\mathbb{E}[\Delta\Delta^\top]
    =\mathbb{E}[\Gamma_\theta(Z)]\succeq0.
    \label{eq:app-covariance-proof}
\end{equation}
For any fixed vector \(v\), the quadratic form of this difference is
\(\mathbb{E}[(v^\top\Delta)^2]\ge0\).
This proves the covariance statement in Proposition~\ref{prop:rb}.
Since both gradients have mean \(\nabla R(\theta)\), taking traces also gives an exact mean-squared error reduction:
\begin{equation}
    \begin{aligned}
        &\mathbb E\|G_{\mathrm H}-\nabla R(\theta)\|_2^2
        -\mathbb E\|G_{\mathrm{RB}}-\nabla R(\theta)\|_2^2\\
        &\qquad=\mathbb E\|G_{\mathrm H}-G_{\mathrm{RB}}\|_2^2
        =\mathbb E\!\left[\operatorname{tr}\Gamma_\theta(Z)\right]\ge0.
    \end{aligned}
    \label{eq:app-gradient-mse}
\end{equation}
The conditional average therefore uses the same prediction to estimate the population gradient with no larger mean-squared error.

For softmax logits \(a_\theta\), let \(J_\theta\in\mathbb R^{K\times d_\theta}\) be their Jacobian and let \(e_Y\) denote the one-hot label.
At a fixed \(S\),
\begin{equation}
    G_{\mathrm{H}}=J_\theta^\top(q_\theta-e_Y),
    \qquad
    G_{\mathrm{RB}}=J_\theta^\top(q_\theta-p_Z).
    \label{eq:app-softmax-gradients}
\end{equation}
Because \(\operatorname{Cov}(e_Y\mid Z)=\operatorname{Diag}(p_Z)-p_Zp_Z^\top\), Equation~\eqref{eq:removable-noise} follows.
The trace of this label covariance is \(1-\|p_Z\|_2^2\).
Its effect in parameter space also depends on \(J_\theta\), and can vanish in directions to which the logits are insensitive.

\subsection{Remaining sources of gradient variation}
\label{app:remaining-noise}

For the special case \(Z=(T,S)\), define \(G_{\mathrm{P}}=\nabla_\theta\ell_{\mathrm{P}}=\mathbb{E}[G_{\mathrm{RB}}\mid S]\).
Applying total covariance first conditional on \(Z\) and then on \(S\) gives
\begin{equation}
    \begin{aligned}
        V_{\mathrm{H}}
        &=V_{\mathrm{episode}}+V_{\mathrm{task}}+V_{\mathrm{label}},\\
        V_{\mathrm{episode}}&:=\operatorname{Cov}(G_{\mathrm{P}}),\\
        V_{\mathrm{task}}&:=\mathbb{E}[\operatorname{Cov}(G_{\mathrm{RB}}\mid S)],\\
        V_{\mathrm{label}}&:=\mathbb{E}[\Gamma_\theta(Z)].
    \end{aligned}
    \label{eq:app-noise-decomposition}
\end{equation}
Conditional supervision removes \(V_{\mathrm{label}}\).
The term \(V_{\mathrm{task}}\) measures variation across mechanisms compatible with the same visible input, while \(V_{\mathrm{episode}}\) remains when the visible input changes.
In particular,
\(V_{\mathrm{H}}\succeq V_{\mathrm{RB}}\succeq\operatorname{Cov}(G_{\mathrm{P}})\).
For a more general \(Z\), the same algebra holds, but the middle term describes variation in all retained generator information conditional on \(S\).

\subsection{Selective conditional supervision}
\label{app:selective-rb}

Let \(A=A(Z)\in\{0,1\}\) be independent of \(\theta\).
With \(G_{\mathrm{M}}=AG_{\mathrm{RB}}+(1-A)G_{\mathrm{H}}\), conditioning on \(Z\) gives
\begin{equation}
    \begin{aligned}
        \mathbb{E}[\ell_{\mathrm{M}}\mid Z]&=\ell_{\mathrm{RB}},\\
        \mathbb{E}[G_{\mathrm{M}}\mid Z]&=G_{\mathrm{RB}},\\
        \operatorname{Cov}(G_{\mathrm{M}}\mid Z)&=(1-A)\Gamma_\theta(Z).
    \end{aligned}
    \label{eq:app-selective-conditional}
\end{equation}
The first identity preserves the population objective.
The other two and total covariance imply
\begin{equation}
    V_{\mathrm{M}}=V_{\mathrm{RB}}+\mathbb{E}[(1-A)\Gamma_\theta(Z)],
    \qquad
    V_{\mathrm{H}}-V_{\mathrm{M}}=\mathbb{E}[A\Gamma_\theta(Z)].
    \label{eq:app-selective-covariance}
\end{equation}
The conditional expectation is defined on every branch, while the implementation evaluates it only when \(A=1\).
The branch indicator is determined by retained state before the residual label draw, as required for the conditional identities above.

\subsection{Multiple queries and independent episodes}
\label{app:multiple-queries}

For a table with \(Q\) queries, let \(S\) include the full model input and write \(q_{\theta,r}\) for the prediction at query \(r\).
Retain a common information set \(Z\) containing \(S\), and define \(p_{Z,r}(k)=P(Y_r=k\mid Z)\).
Linearity of conditional expectation gives
\begin{equation}
    \begin{aligned}
        \ell_{\mathrm{H}}^{\mathrm{table}}
        &:=-\frac1Q\sum_{r=1}^{Q}\log q_{\theta,r}(Y_r\mid S),\\
        \mathbb{E}[\ell_{\mathrm{H}}^{\mathrm{table}}\mid Z]
        &=-\frac1Q\sum_{r=1}^{Q}\sum_{k=1}^{K}
            p_{Z,r}(k)\log q_{\theta,r}(k\mid S).
    \end{aligned}
    \label{eq:app-table-loss}
\end{equation}
Hence Proposition~\ref{prop:rb} applies to the whole table without assuming conditional independence between its labels.
With row Jacobians \(J_{\theta,r}\), the corresponding conditional covariance is
\begin{equation}
    \Gamma_\theta^{\mathrm{table}}(Z)
    =\frac{1}{Q^2}\sum_{r=1}^{Q}\sum_{s=1}^{Q}
        J_{\theta,r}^{\top}
        \operatorname{Cov}(e_{Y_r},e_{Y_s}\mid Z)
        J_{\theta,s}.
    \label{eq:app-table-covariance}
\end{equation}
The off-diagonal terms vanish under conditional independence.
Averaging \(m\) independent table gradients divides their covariance by \(m\).
Within a table, Equation~\eqref{eq:app-table-covariance} accounts for both diagonal and cross-query contributions.

\subsection{Asymptotic parameter covariance}
\label{app:asymptotic}

We relate the fixed-parameter comparison to the finite-dimensional M-estimation setting of Theorem B.1 in the PFN analysis~\cite{nagler2023statistical}.
Assume a fixed parameter dimension, a unique interior minimizer \(\theta^\star\), and a positive-definite population Hessian \(H_\star=\nabla^2R(\theta^\star)\).
For \(a\in\{\mathrm{H},\mathrm{RB}\}\), let \(\widehat\theta_{a,m}\) be an empirical minimizer based on \(m\) independent episodes.
Assume these minimizers are consistent, are interior with probability tending to one, and have empirical Hessians converging locally uniformly to the continuous population Hessian.
Together with the moment and differentiation conditions above, a Taylor expansion of each empirical score gives
\begin{equation}
    \sqrt m(\widehat\theta_{a,m}-\theta^\star)
    =-H_\star^{-1}\frac1{\sqrt m}\sum_{j=1}^{m}
        G_a^{(j)}(\theta^\star)+o_P(1).
    \label{eq:app-estimator-expansion}
\end{equation}
The central limit theorem yields
\begin{equation}
    \sqrt m(\widehat\theta_{a,m}-\theta^\star)
    \ \Rightarrow\ \mathcal N(0,\Sigma_a),
    \qquad
    \Sigma_a=H_\star^{-1}V_a(\theta^\star)H_\star^{-\top}.
    \label{eq:app-asymptotic-normality}
\end{equation}
The two objectives have the same population Hessian.
Applying Equation~\eqref{eq:app-covariance-proof} at \(\theta^\star\) therefore gives
\begin{equation}
    \Sigma_{\mathrm{H}}-\Sigma_{\mathrm{RB}}
    =H_\star^{-1}\mathbb{E}[\Gamma_{\theta^\star}(Z)]H_\star^{-\top}
    \succeq0.
    \label{eq:app-asymptotic-ordering}
\end{equation}
Under these M-estimation assumptions, the empirical minimizers share the \(m^{-1/2}\) rate, and conditional supervision reduces the covariance of the limiting parameter distribution.

% !TeX root = ../simrb.tex
\section{Conditional Targets in the Generator}
\label{app:generator-details}

\subsection{Observed features and task acceptance}
\label{app:selection}

Let \(\mathcal H_0\) denote retained state before a query-label draw, with
\(p^0_k=P(Y=k\mid\mathcal H_0=h)\), and let \(E\in\{0,1\}\) indicate task acceptance.
Writing \(L_k(h,s)\) for the conditional likelihood of \(S=s,E=1\) given \(\mathcal H_0=h,Y=k\), Bayes' rule gives
\begin{equation}
    P(Y=k\mid\mathcal H_0=h,S=s,E=1)
    =\frac{p^0_k L_k(h,s)}{\sum_j p^0_j L_j(h,s)}.
    \label{eq:app-acceptance-posterior}
\end{equation}
Thus the probabilities before sampling remain valid if the observed input and acceptance have a likelihood independent of the label on its support.
This condition holds for a terminal draw when input generation and acceptance no longer read that draw.
In the terminal construction, context-based class boundaries and label checks are part of the generator shared by hard-label and RB training.

When features depend on the sampled label, the likelihood term can change the target.
For example, if \(Y\) is uniform on \(\{0,1\}\) and \(X=Y\), observing \(X=1\) makes the conditional label deterministic.
The likelihood then updates the pre-draw probabilities \((1/2,1/2)\) to a point mass at the observed class.
This example separates the probabilities used to generate a label from the probabilities conditional on the final model input.
The retained state includes that input and leaves the label randomness to be averaged unobserved; storing a random seed that determines the query label would instead fix its value.

\subsection{Posterior under the Gaussian channel}
\label{app:channel-posterior}

We first use the channel law before filtering and then condition on acceptance.
Fix \(\mathcal H=h\), including the codebook, all row probabilities, and a scale \(\sigma(h)>0\) determined solely by this pre-draw state.
Thus the scale is measurable with respect to \(\mathcal H\), irrespective of when the deterministic calibration routine is executed.
Let \(\phi_\sigma(w-b)\) be the density of \(\mathcal N(b,\sigma^2I_d)\).
Conditional independence of row labels and independent encoding noise imply the joint density
\begin{equation}
    f_h(y_{1:L},w_{1:L})
    =\prod_{r=1}^{L}p_{r,y_r}\phi_\sigma(w_r-b_{y_r}),
    \label{eq:app-channel-joint}
\end{equation}
where \(f_h\) is the joint density with respect to counting measure for labels and Lebesgue measure for encodings.
Conditioning on all encodings and the hard context labels gives
\begin{equation}
    P(Y_{N+1:L}=y_{N+1:L}\mid\mathcal H,W_{1:L},Y_{1:N})
    =\prod_{r=N+1}^{L}
        \frac{p_{r,y_r}\phi_\sigma(W_r-b_{y_r})}
        {\sum_j p_{rj}\phi_\sigma(W_r-b_j)}.
    \label{eq:app-channel-factorization}
\end{equation}

The required downstream condition is
\begin{equation}
    Y_{N+1:L}\ \perp\ (S,E)
    \ \mid\ (\mathcal H,W_{1:L},Y_{1:N}).
    \label{eq:app-channel-independence}
\end{equation}
This condition holds when downstream inputs and acceptance are computed from the conditioning variables and auxiliary randomness that is conditionally independent of the query labels.
To see the computational consequence, fix \((h,w_{1:L},y_{1:N})\) and denote the likelihood of \(S=s,E=1\) by \(\Lambda(s,1\mid h,w_{1:L},y_{1:N})\).
By Equation~\eqref{eq:app-channel-independence}, this likelihood is the same for every candidate query-label vector.
Multiplying Equation~\eqref{eq:app-channel-factorization} by \(\Lambda\) and normalizing therefore cancels the entire downstream likelihood.
With \(Z_\sigma=(\mathcal H,W_{1:L},Y_{1:N},S)\), Equation~\eqref{eq:app-channel-independence} implies, on accepted tasks,
\begin{equation}
    P(Y_r=k\mid Z_\sigma,E=1)
    =P(Y_r=k\mid\mathcal H,W_r)
    =\widetilde p_{\sigma,r}(k).
    \label{eq:app-accepted-channel-posterior}
\end{equation}
The Gaussian normalizing factors cancel because all classes share \(\sigma\), giving Equation~\eqref{eq:channel-posterior}.
Conditioning on \(S\) and acceptance also preserves the product posterior in Equation~\eqref{eq:app-channel-factorization}.

The target uses the likelihood of the actual channel.
For a scalar encoding broadcast into several feature positions, this likelihood is evaluated on the single noisy scalar before broadcasting; repeated copies are not independent observations.
Posterior weights can be computed in log space from
\(\log p_{rk}-\|W_r-b_k\|_2^2/(2\sigma^2)\) and normalized using log-sum-exp.
Classes with zero prior probability retain zero posterior probability.

\subsection{Objective preservation after calibration and fallback}
\label{app:calibrated-objective}

Fix a calibration rule and a decision to enable the channel, both based on state available before the query-label and encoding-noise draws.
These choices, the original branches retained as fallback, and the acceptance procedure define \(P_{\mathrm{cal}}\), an instance of the law \(P\) in Section~\ref{sec:conditional-supervision}.
All conditional distributions in this paragraph are taken under that law.

On enabled branches satisfying Equation~\eqref{eq:app-channel-independence}, Equation~\eqref{eq:app-accepted-channel-posterior} supplies the exact query marginals.
Linearity of conditional expectation then gives
\begin{equation}
    \mathbb{E}_{\mathrm{cal}}\!\left[
        -\frac1Q\sum_{r=N+1}^{L}\log q_{\theta,r}(Y_r\mid S)
        \,\middle|\,Z_\sigma
    \right]
    =-\frac1Q\sum_{r=N+1}^{L}\sum_k
        \widetilde p_{\sigma,r}(k)\log q_{\theta,r}(k\mid S).
    \label{eq:app-calibrated-loss}
\end{equation}
The generator state used to combine branches includes the branch decision.
Applying Proposition~\ref{prop:rb} on exact conditional-target branches and Equation~\eqref{eq:selective-rb} for hard-label fallback gives
\begin{equation}
    R_{\mathrm{cal,H}}(\theta)=R_{\mathrm{cal,RB}}(\theta),
    \qquad
    V_{\mathrm{cal,H}}-V_{\mathrm{cal,RB}}\succeq0,
    \label{eq:calibrated-rb}
\end{equation}
Here the RB procedure includes hard-label fallback on untreated branches, and both estimators target the population objective of \(P_{\mathrm{cal}}\).

\subsection{Instantiation in GraphSCM converters}
\label{app:converter-mapping}

We apply the preceding rules to the categorical converters of the TabICLv2 generator~\cite{qu2026tabiclv2}.
In the softmax converters, a continuous node state determines class probabilities, from which a discrete class is sampled.
The class and the state passed to downstream nodes are separate outputs of the converter.
The \texttt{softmax\_id} mode passes on the original continuous state, so the sampled class does not enter subsequent feature computations.
With query-label-independent acceptance, its pre-draw probabilities provide terminal conditional targets even when the node has observed descendants.

The \texttt{softmax\_disc} mode passes on the sampled class's vector code, and \texttt{softmax\_int} passes on its scalar integer code.
On these modes, Equation~\eqref{eq:gaussian-channel} supplies the state used by the downstream graph.
When the sampled class can reach an observed feature, supervision uses the posterior in Equation~\eqref{eq:channel-posterior}.
When the final input and acceptance are independent of the query-category draw conditional on the pre-draw state and context labels, the encoding can be omitted from the retained information and the pre-draw probabilities remain available.
An explicitly infeasible entropy budget selects the sampled hard target as described in Appendix~\ref{app:entropy-calibration}.
Nearest-neighbor converters assign classes deterministically given the node state and centers, and retain one-hot targets.

The same Gaussian scale is used for encoding and posterior evaluation.
A scalar code receives one noise draw per row before broadcasting, and class permutations are applied to both labels and probability columns.
The context supplied to the model keeps its sampled hard labels throughout.

% !TeX root = ../simrb.tex
\section{Entropy Bound and Scale Calibration}
\label{app:entropy-details}

\subsection{From mixture entropy to conditional label entropy}
\label{app:entropy-bound}

Under the channel distribution before task filtering, fix the state \(\mathcal H=h\), a row \(r\), and a positive scale \(\sigma\).
Use \(h_{\mathrm{diff}}\) for differential entropy and write
\(h_\sigma=\tfrac d2\log(2\pi e\sigma^2)\) for the entropy of each Gaussian component.
The discrete-continuous entropy identity gives
\begin{equation}
    \mathcal E_r(\sigma)
    =H(p_r)+h_\sigma-h_{\mathrm{diff}}(W_r\mid\mathcal H=h).
    \label{eq:app-entropy-identity}
\end{equation}
For the mixture \(\sum_i p_{ri}\mathcal N(b_i,\sigma^2I_d)\), Equation (12) in the corrected arXiv version of the mixture-entropy analysis~\cite{kolchinsky2017mixture} gives the Bhattacharyya lower bound
\begin{equation}
    h_{\mathrm{diff}}(W_r\mid\mathcal H=h)
    \ge h_\sigma-\sum_i p_{ri}\log\!\left[
        \sum_j p_{rj}\exp\!\left(-\frac{\|b_i-b_j\|_2^2}{8\sigma^2}\right)
    \right].
    \label{eq:app-mixture-entropy-bound}
\end{equation}
Section IV of that version gives the Gaussian pairwise distances.
Substitution into Equation~\eqref{eq:app-entropy-identity} proves \(\mathcal E_r(\sigma)\le U_r(\sigma)\).

For every \(i\) with \(p_{ri}>0\), the inner sum lies in \([p_{ri},1]\).
Its logarithm therefore lies in \([\log p_{ri},0]\), which proves
\(0\le U_r(\sigma)\le H(p_r)\).
Together with the nonnegativity of conditional label entropy, this yields Equation~\eqref{eq:entropy-upper-bound}.
Zero-probability terms are interpreted by continuity.

\subsection{Monotonicity, duplicate codes, and feasibility}
\label{app:entropy-calibration}

Each factor \(\exp[-\|b_i-b_j\|_2^2/(8s^2)]\) is continuous and nondecreasing for \(s>0\).
Positive weighted sums and logarithms preserve this ordering, so \(U_r(s)\) and \(\overline U(s)\) are continuous and nondecreasing on that domain.
Their limits at zero exist.
For a code shared by several classes, let
\(\rho_{ri}=\sum_{j:b_j=b_i}p_{rj}\).
Then
\begin{equation}
    U_r(0^+)
    =H(p_r)+\sum_{i:p_{ri}>0}p_{ri}\log\rho_{ri}
    =H(Y_r\mid b_{Y_r},\mathcal H=h).
    \label{eq:app-zero-noise-bound}
\end{equation}
Distinct codes give \(U_r(0^+)=0\).
Duplicate codes can leave positive uncertainty even without encoding noise.
Let \(\mathcal B_\beta(h)\) denote either budget policy in Appendix~\ref{app:budget-policies}.
If \(\overline U(0^+)>\mathcal B_\beta(h)\), the budget is infeasible; equality can also leave zero as the only feasible scale.
An infeasible budget retains the original encoding and uses the sampled one-hot target, keeping the task subject to the same acceptance rules.
The positive-channel entropy guarantee applies to branches on which calibration returns a feasible \(\sigma>0\).

For \(\overline H>0\), calibration first checks the zero-scale limit and any finite cap.
A feasible finite cap is the maximum in Equation~\eqref{eq:entropy-scale}.
Otherwise, bisection maintains a feasible lower endpoint and an infeasible upper endpoint.
For an uncapped search, we use \(\beta<1\), so \(\mathcal B_\beta(h)<\overline H\) under either policy.
Since \(\overline U(s)\to\overline H\) as \(s\to\infty\), repeatedly doubling a positive endpoint initialized from a nonzero code distance eventually brackets the budget threshold.
The implementation uses 24 bisection steps and returns the feasible lower endpoint, preserving the constraint in exact arithmetic.

Effectively deterministic tables and searches that resolve no positive scale retain the original encoding.
Their supervision follows the branch rule: use terminal probabilities when the final input and acceptance are independent of the remaining query draw, and sampled hard labels otherwise.
This is separate from the explicit one-hot rule for an infeasible budget.
With \(\beta=1\), the implementation requires a finite cap to define a finite scale choice for every task.

The distance matrix and entropy bound are evaluated in double precision, with zero-probability terms handled by continuity.
The distance matrix is computed once, at cost \(O(K^2d)\), and each bound evaluation costs \(O(LK^2)\); rows can be processed in chunks.
Calibration uses \(\frac1L\sum_r U_r(s)\), which generally differs from evaluating the bound at the mean row probabilities.
The selected scale is used consistently in both the Gaussian draw and its posterior likelihood.

\subsection{Average information guarantee and its scope}
\label{app:information-budget}

For any feasible positive scale selected at \(\mathcal H=h\), averaging Equation~\eqref{eq:entropy-upper-bound} and applying the feasibility constraint in Equation~\eqref{eq:entropy-scale} gives
\begin{equation}
    \frac1L\sum_{r=1}^{L}\mathcal E_r(\sigma_\star)
    \le\overline U(\sigma_\star)
    \le\mathcal B_\beta(h)\le\beta\overline H.
    \label{eq:app-average-entropy-budget}
\end{equation}
Using \(I(Y_r;W_r\mid\mathcal H=h)=H(p_r)-\mathcal E_r(\sigma_\star)\) also bounds the retained mutual information:
\begin{equation}
    \begin{aligned}
        \frac1L\sum_{r=1}^{L}H(Y_r\mid W_r,\mathcal H=h)
        &\le\mathcal B_\beta(h),\\
        \frac1L\sum_{r=1}^{L}I(Y_r;W_r\mid\mathcal H=h)
        &\ge\overline H-\mathcal B_\beta(h)
        \ge(1-\beta)\overline H.
    \end{aligned}
    \label{eq:local-information-budget}
\end{equation}
Here the channel uses \(\sigma_\star(h)\).

The averaging is over all \(L\) rows and over the label and encoding draws conditional on the state before task filtering.
Thus Equation~\eqref{eq:local-information-budget} specifies an all-row, pre-selection guarantee for the local encoding channel.
Within this constraint, calibration selects the largest certified scale; the resulting conditional supervision satisfies Equation~\eqref{eq:calibrated-rb}.

\subsection{Relative and binary-cap budgets}
\label{app:budget-policies}

The two implemented budget policies differ only in the allowed residual entropy:
\begin{equation}
    \mathcal B_\beta(h)=
    \begin{cases}
        \beta\overline H, & \text{relative},\\
        \beta\min\{\overline H,\log 2\}, & \text{binary cap}.
    \end{cases}
    \label{eq:app-budget-policies}
\end{equation}
The cap is applied after averaging the row entropies.
For binary tasks, \(\overline H\le\log 2\), so both policies give the same budget and scale choice.
For multiclass tasks with \(\overline H>\log 2\), the binary cap limits the absolute residual-entropy allowance to \(\beta\log 2\).
Under either policy, the same upper bound, posterior, and root-finding procedure apply.
Since \(\mathcal B_\beta(h)\le\beta\overline H\), both satisfy the information guarantee in Equation~\eqref{eq:local-information-budget}.

The implementation's \texttt{target\_encoding\_noise\_max\_std} value of zero denotes the uncapped choice \(\sigma_{\max}=\infty\).
The solver still determines a task-dependent positive scale whenever the budget and numerical resolution permit it.
The all-row policy described here computes one budget from all row probabilities and applies the resulting scale to both context and query encodings.

% !TeX root = ../simrb.tex
\section{Experimental Details}
\label{app:experiment-details}

\subsection{Pretraining and checkpoints}

All six TFM conditions are pretrained for this study; Hard uses the original GraphSCM generator and hard query labels.
The implementations have 27.55M parameters for TabICLv2, 8.20M for TabSwift, and 10.72M for TabPFN-2.5.
The TabPFN-2.5 implementation uses 24 layers, a hidden dimension of 192, and three attention heads.
Within each architecture, Hard and \method{} use the same effective global batch sizes and learning-rate schedule.
Training follows the two-stage protocol in Section~\ref{sec:experiment-setup}.
Stage~2 starts from the corresponding condition's own Stage~1 model, so the downstream comparisons measure accumulated training effects.
Checkpoint selection is fixed across datasets.

\method{} uses a task-level entropy fraction of 0.15, all-row averaging, the binary-cap policy of Appendix~\ref{app:budget-policies}, selective conditional targets, and no external Gaussian-scale cap.
The model architecture and inference configuration are matched within each Hard/\method{} pair.

\subsection{Benchmark protocols and aggregation}
\label{app:evaluation-protocols}

\paragraph{TabArena.}
We evaluate 30 binary and eight multiclass datasets over all 594 official outer splits in the archived protocol.
The TFM runs use the default benchmark wrappers, eight-fold internal validation, and full-data refitting.
TabICLv2 uses an ensemble size of eight, TabSwift uses 16, and TabPFN-2.5 uses the archived wrapper defaults.
These internal procedures are distinct from the official outer test splits used to compute scores.
The five conventional baselines reuse the benchmark's tuned single-model results on the same outer splits.

\paragraph{TALENT.}
For each dataset, the original train, validation, and test partitions are combined and repartitioned with stratified five-fold cross-validation using a fixed seed.
TFMs use an ensemble size of eight.
Baseline hyperparameters are transferred from available TALENT configurations or selected with at most 30 trials or 20 minutes per dataset and method.
Additional searches for GBDTs, RealMLP, and TabM optimize binary ROC-AUC or multiclass log loss using the original train/validation split and two additional stratified holdouts from the combined training and validation data.
The selected configuration is reused across all five outer folds, with stopping iterations selected within each outer training fold before refitting.
This transferred-hyperparameter protocol is not nested cross-validation.
Log loss is computed from normalized class probabilities in FP64, with probabilities clipped below at $10^{-15}$.

The main comparison contains \TalentMainDatasets{} datasets shared by all methods, with \TalentMainFolds{} observed folds per method.
The extended paired comparisons use each architecture's complete supported set.

\paragraph{BeyondArena.}
The comparison with external baselines uses the 12 complete datasets shared by all methods: six binary and six multiclass tasks, comprising four temporal and eight grouped datasets.
All 77 outer splits have at most 100K training rows.
We use the official outer splits and default TFM preprocessing wrappers with no inner cross-validation; the reused tuned baselines retain their benchmark training and validation procedures.
The data revision is \texttt{2ecfe882ccfb814fc27c4de10a64ceefd5d7655c}.

The broader non-IID classification collection contains 21 datasets and 92 outer splits.
The complete paired scope covers all 21 datasets for TabICLv2, 18 for TabSwift, and 12 for TabPFN-2.5.
The larger-table runs use the ensemble settings recorded in their evaluation configurations, including ensemble size one for TabSwift above 100K training rows.
Only datasets with every official split completed for both conditions enter a paired comparison.
Incomplete or cancelled evaluations receive no imputed score.
The main-text OOD win rates in Table~\ref{tab:paired-win-rates} use these architecture-specific sets.
Table~\ref{tab:extended-paired} reports their win/loss/tie counts, together with the extended IID ranges.

% Generated by scripts/build_experiment_assets.py.
\begin{table}[!t]
\centering
\caption{Stage~2 paired results on each architecture's complete supported range. These ranges differ across architectures and are not used for cross-architecture rankings. W/L/T denotes wins/losses/ties of \method{} against Hard.}
\label{tab:extended-paired}
\small
\setlength{\tabcolsep}{4pt}
\begin{tabular}{@{}llrrr@{}}
\toprule
Benchmark & Architecture & Datasets & W/L/T & Win (\%) \\
\midrule
TabArena & TabICLv2 & 38 & 21/17/0 & 55.3 \\
TabArena & TabSwift & 38 & 25/13/0 & 65.8 \\
TabArena & TabPFN-2.5 & 38 & 27/11/0 & 71.1 \\
\addlinespace
TALENT & TabICLv2 & 200 & 132/65/3 & 66.0 \\
TALENT & TabSwift & 200 & 137/62/1 & 68.5 \\
TALENT & TabPFN-2.5 & 188 & 101/85/2 & 53.7 \\
\addlinespace
BeyondArena & TabICLv2 & 21 & 11/10/0 & 52.4 \\
BeyondArena & TabSwift & 18 & 15/3/0 & 83.3 \\
BeyondArena & TabPFN-2.5 & 12 & 7/5/0 & 58.3 \\
\addlinespace
\bottomrule
\end{tabular}
\end{table}

\paragraph{Metrics.}
Within a benchmark and training stage, all 11 methods are ranked on the same dataset/split set.
For method $m$ on split $s$ of dataset $d$, define $e_{m,d,s}$ as $1-\mathrm{AUC}$ for binary tasks or log loss for multiclass tasks.
The split-level version of TabArena's Improvability~\cite{erickson2025tabarena} is
\begin{equation}
    I_{m,d,s}=1-\frac{\min_j e_{j,d,s}}{e_{m,d,s}},
    \qquad I_{m,d,s}=0\ \text{if }e_{m,d,s}=0.
    \label{eq:experiment-improvability}
\end{equation}
We average split-level ranks and Improvability within each dataset, then average across datasets with equal weight.
This prevents datasets with more outer splits from receiving greater aggregate weight.
The best-error reference and the ranks are computed separately for each stage's method set.
The five conventional baselines retain the same predictions at both stages.

Paired wins instead compare the mean raw metric within each dataset.
Higher AUC or lower log loss counts as a win; differences within $10^{-12}$ count as ties.
Win rate is $W/(W+L+T)$.
Mean rank, Improvability, and win rate summarize different aspects of performance and need not agree in direction.

\subsection{Detailed OOD results}
\label{app:ood-results}

Table~\ref{tab:ood-results} gives the raw metrics and aggregate results on the 12 common BeyondArena datasets used in Table~\ref{tab:ood-wins}.
The two win-rate columns in Table~\ref{tab:ood-wins} use each architecture's paired outcomes: Hard wins are \method{} losses, while ties remain in both denominators.
At Stage~2 there are no ties on this common set.
\method{} lowers mean Improvability for all three architectures; TabSwift also improves mean rank.
TabM has the lowest Improvability among the 11 methods, while the Hard TabPFN-2.5 control has the lowest mean rank.

\begin{table}[!t]
    \centering
    \caption{Paired OOD win rates on 12 BeyondArena datasets.
    Each column reports its dataset-level wins against the other condition. There are no ties.}
    \label{tab:ood-wins}
    \begin{minipage}{0.48\linewidth}
        % Generated by scripts/build_experiment_assets.py.
\begingroup
\fontsize{9}{10.5}\selectfont
\setlength{\tabcolsep}{2.5pt}
\renewcommand{\arraystretch}{1.18}
\begin{tabular*}{\linewidth}{@{\extracolsep{\fill}}lrr@{}}
\toprule
Architecture & Hard & \method{} \\
\midrule
TabICLv2 & 50.0\% & 50.0\% \\
TabSwift & 25.0\% & \textbf{75.0\%} \\
TabPFN-2.5 & 41.7\% & \textbf{58.3\%} \\
\bottomrule
\end{tabular*}
\endgroup

    \end{minipage}
\end{table}

% Generated by scripts/build_experiment_assets.py.
\begin{table}[!t]
\centering
\caption{Stage~2 OOD results on the 12 common BeyondArena datasets (six binary, six multiclass). AUC is binary ROC-AUC; LL is multiclass log loss; $I$ is mean Improvability (\%); Rank is mean rank among 11 methods. Scores are averaged within each dataset and then equally across datasets. Win compares \method{} with its same-architecture Hard control. Bold marks the best AUC, LL, $I$, and Rank.}
\label{tab:ood-results}
\small
\setlength{\tabcolsep}{6pt}
\begin{tabular}{@{}lrrrrr@{}}
\toprule
Method & AUC$\uparrow$ & LL$\downarrow$ & $I\downarrow$ & Rank$\downarrow$ & Win (\%) \\
\midrule
CatBoost & 0.7806 & 0.8537 & 20.25 & 4.85 & \textemdash \\
LightGBM & 0.7830 & 0.8638 & 21.42 & 6.06 & \textemdash \\
XGBoost & 0.7791 & 0.9670 & 26.63 & 6.59 & \textemdash \\
RealMLP & 0.7943 & 0.8508 & 18.64 & 5.37 & \textemdash \\
TabM & 0.7894 & 0.8133 & \textbf{16.75} & 4.87 & \textemdash \\
\midrule
TabICLv2 (Hard) & 0.7527 & 0.8655 & 27.45 & 7.21 & \textemdash \\
TabICLv2 + \method{} & 0.7598 & 0.8557 & 27.33 & 7.26 & 50.0 \\
\addlinespace[2pt]
TabSwift (Hard) & 0.7716 & 0.9376 & 27.12 & 7.81 & \textemdash \\
TabSwift + \method{} & 0.7771 & 0.9304 & 25.05 & 7.11 & 75.0 \\
\addlinespace[2pt]
TabPFN-2.5 (Hard) & \textbf{0.8031} & 0.8064 & 18.31 & \textbf{4.19} & \textemdash \\
TabPFN-2.5 + \method{} & 0.8002 & \textbf{0.7880} & 17.95 & 4.68 & 58.3 \\
\bottomrule
\end{tabular}
\end{table}

\subsection{Gradient diagnostics}
\label{app:gradient-diagnostics}

The main grid fixes weights from the \method{} training trajectory at initialization, Stage~1 1K/10K/20K/30K, and Stage~2 3K/6K.
Each checkpoint is evaluated at total table lengths $L=1024$ and $8192$, yielding 42 settings.
The corresponding context/query counts are $819/205$ and $6553/1639$.
The archived field named \texttt{context} records $L$, rather than the number of labeled context rows.
Hard and \method{} share the model forward pass, query mask, and equal-table/equal-query loss normalization.
Dropout is disabled, and no optimizer step, weight decay update, or gradient scaler is applied.

For each setting, we draw $P=32$ independent pairs of batches, with 16 tables per batch.
Writing $G_a^{i,1}$ and $G_a^{i,2}$ for the two independent batch gradients under supervision $a\in\{\mathrm H,\mathrm G\}$ (Hard or \method{}), we clip each at global norm one:
$\widetilde G_a^{i,j}=G_a^{i,j}\min\{1,(\|G_a^{i,j}\|_2+10^{-6})^{-1}\}$.
The primary diagnostic uses
\begin{equation}
    \widehat v_a=\frac{1}{2P}\sum_{i=1}^{P}
    \left\|\widetilde G_a^{i,1}-\widetilde G_a^{i,2}\right\|_2^2,
    \qquad
    \widehat\rho=1-\frac{\widehat v_{\mathrm G}}{\widehat v_{\mathrm H}}.
    \label{eq:gradient-diagnostic-estimator}
\end{equation}
Each $\widehat v_a$ estimates the trace of the clipped batch-gradient covariance.
We resample complete batch pairs 4,000 times, keeping both supervision rules paired, to obtain pointwise 95\% bootstrap intervals for $\widehat\rho$ that quantify batch-sampling variability at fixed weights.

\paragraph{Task populations.}
The main grid samples tasks with tractable conditional query targets from the calibrated GraphSCM mixture, as determined by converter type, pre-draw calibration, and target-to-feature topology.
This selection does not depend on sampled query labels, loss, or gradients.
These tasks use terminal probabilities or Gaussian-channel posteriors and are labeled Tractable targets in Table~\ref{tab:gradient-controls}.
The table additionally evaluates the full calibrated mixture, channel-only tasks, and the same task set restricted to 128 supervised queries at Stage~2 6K, yielding 18 control settings.
Full-mixture query coverage is approximately 26.2\% at $L=1024$ and 30.7\% at $L=8192$; the remainder is largely hard-label fallback.
The diagnostic uses fixed table lengths and independent shape draws, so its batch distribution differs from the complete pretraining stream.
All Stage~2 6K diagnostics use the same final decay checkpoints as predictive evaluation.

% Generated by scripts/build_experiment_assets.py.
\begin{table}[!t]
\centering
\caption{Gradient diagnostics at the final checkpoints after global norm-one clipping of each whole-batch gradient. RB coverage is the fraction of supervised query rows using conditional targets. Intervals are pointwise 95\% batch-pair bootstrap intervals.}
\label{tab:gradient-controls}
\small
\setlength{\tabcolsep}{5pt}
\begin{tabular}{@{}llrcr@{}}
\toprule
Architecture & Pool & $L$ & Reduction (\%) & RB (\%) \\
\midrule
TabICLv2 & Full & 1,024 & $+1.37\;[-0.44,+3.23]$ & 26.2 \\
TabICLv2 & Channel & 1,024 & $+2.07\;[-0.77,+5.06]$ & 98.5 \\
TabICLv2 & Tractable targets & 1,024 & $+11.52\;[+6.72,+15.93]$ & 99.3 \\
TabICLv2 & Tractable targets, $Q=128$ & 1,024 & $+13.77\;[+9.50,+18.32]$ & 99.4 \\
TabICLv2 & Full & 8,192 & $+1.69\;[+0.81,+3.07]$ & 30.7 \\
TabICLv2 & Channel & 8,192 & $+1.99\;[-0.81,+5.68]$ & 98.3 \\
TabICLv2 & Tractable targets & 8,192 & $+6.80\;[+3.16,+11.24]$ & 99.2 \\
TabICLv2 & Tractable targets, $Q=128$ & 8,192 & $+18.64\;[+13.77,+23.76]$ & 99.2 \\
\addlinespace
TabSwift & Full & 1,024 & $+3.37\;[+2.36,+4.61]$ & 26.2 \\
TabSwift & Channel & 1,024 & $+6.82\;[+4.51,+9.32]$ & 98.5 \\
TabSwift & Tractable targets & 1,024 & $+20.91\;[+18.06,+24.03]$ & 99.3 \\
TabSwift & Tractable targets, $Q=128$ & 1,024 & $+22.71\;[+19.08,+26.28]$ & 99.4 \\
TabSwift & Full & 8,192 & $-0.002\;[-0.99,+0.90]$ & 30.7 \\
TabSwift & Channel & 8,192 & $+6.36\;[+2.08,+12.76]$ & 98.3 \\
TabSwift & Tractable targets & 8,192 & $+10.17\;[+4.53,+17.15]$ & 99.2 \\
TabSwift & Tractable targets, $Q=128$ & 8,192 & $+27.94\;[+22.21,+34.13]$ & 99.2 \\
\addlinespace
TabPFN-2.5 & Full & 1,024 & $+7.33\;[+5.78,+8.97]$ & 26.2 \\
TabPFN-2.5 & Channel & 1,024 & $+9.52\;[+8.26,+10.79]$ & 98.5 \\
TabPFN-2.5 & Tractable targets & 1,024 & $+26.04\;[+23.88,+28.30]$ & 99.3 \\
TabPFN-2.5 & Tractable targets, $Q=128$ & 1,024 & $+27.22\;[+24.87,+29.71]$ & 99.4 \\
TabPFN-2.5 & Full & 8,192 & $+4.13\;[+1.76,+8.76]$ & 30.7 \\
TabPFN-2.5 & Channel & 8,192 & $+4.57\;[+3.27,+5.94]$ & 98.3 \\
TabPFN-2.5 & Tractable targets & 8,192 & $+19.30\;[+15.19,+23.71]$ & 99.2 \\
TabPFN-2.5 & Tractable targets, $Q=128$ & 8,192 & $+34.41\;[+29.31,+39.80]$ & 99.2 \\
\addlinespace
\bottomrule
\end{tabular}
\end{table}

% !TeX root = ../simrb.tex
\FloatBarrier
\subsection{Functional components after clipping}
\label{app:module-analysis}

\paragraph{Additive decomposition.}
We partition the trainable parameter coordinates after applying one global clipping factor to each whole-batch gradient.
The estimate $\widehat v_{a,m}$ is obtained by restricting the squared distances in Equation~\eqref{eq:gradient-diagnostic-estimator} to group $m$.
The squared Euclidean norm and the covariance trace add over disjoint coordinate groups, so $\widehat v_a=\sum_m\widehat v_{a,m}$ without assuming independence between groups.
We report the Hard variance share $s_m$, within-component relative reduction $r_m$, and contribution to the whole-model reduction $c_m$:
\begin{equation}
    s_m=\frac{\widehat v_{\mathrm H,m}}{\widehat v_{\mathrm H}},\qquad
    r_m=100\left(1-\frac{\widehat v_{\mathrm G,m}}{\widehat v_{\mathrm H,m}}\right),\qquad
    c_m=100\frac{\widehat v_{\mathrm H,m}-\widehat v_{\mathrm G,m}}{\widehat v_{\mathrm H}}.
    \label{eq:variance-components}
\end{equation}
Here $r_m$ is defined for $\widehat v_{\mathrm H,m}>0$ and expressed as a percentage; $c_m$ is a percentage-point contribution to the whole-model reduction $R$.
The contributions satisfy
\begin{equation}
    c_m=s_mr_m,\qquad
    \sum_m c_m=R=100\left(1-\frac{\widehat v_{\mathrm G}}{\widehat v_{\mathrm H}}\right).
    \label{eq:module-contribution}
\end{equation}
We compute ratios after summing and averaging the unrounded batch-pair distances.
All groups and both supervision rules use the same 4,000 resamples of complete batch pairs.
Within each resample, group variances add to the full-model variance, shares sum to one, and contributions sum to the net reduction.
Tables~\ref{tab:variance-components-1024} and~\ref{tab:variance-components-8192} give all three statistics at 36K.
At these final checkpoints, the percentage of the net reduction contributed by components outside the prediction head is $100(R-c_{\mathrm{head}})/R$.

A component's contribution depends on both its Hard variance share and its relative reduction.
For example, at 36K and $L=1024$, TFicl has a larger relative reduction than TFcol, 26.85\% versus 9.05\%, yet contributes less to the whole-model reduction, 2.21 versus 7.87 percentage points.
TFcol accounts for 86.95\% of Hard variance, compared with 8.23\% for TFicl.

\paragraph{Functional partitions.}
TabICLv2 separates three backbone components and the prediction head.
TFcol includes the input projection, its label encoder, and all column-encoder blocks, including their inducing vectors and SSMax networks.
TFrow includes the row-interaction blocks, row CLS tokens, and output normalization.
TFicl includes the 12 ICL blocks, its label encoder, and output normalization.
The prediction decoder forms a separate group, shown as Pred.\ head in Figure~\ref{fig:gradient-stages}.

TabSwift uses row-wise attention (Row attn in Figure~\ref{fig:gradient-stages}), MLP, Encoders, the prediction head, and Other.
Its attention operates across row tokens, with each row embedded as one token and additional learned register tokens.
This group includes query, key, value, and output projections together with attention gates; MLP contains the block feed-forward layers.
Encoders contains feature and label projections, the prediction head contains the decoder, and Other contains normalization parameters and register tokens.
TabPFN-2.5 separates within-row attention, operating across features of one sample, from within-column attention, operating across samples in one feature column.
Its other groups are MLP, Encoders (feature, positional, and label encodings), the prediction head, and Other (learned thinking rows).
Membership follows parameter paths and is fixed across checkpoints and table lengths.

\paragraph{Parameter coverage.}
Every trainable parameter is assigned to exactly one component, with coverage checked across all 42 main diagnostic settings.
These statistics sum gradient variance over scalar coordinates without normalization by parameter count.

\begin{table}[!t]
\centering
\caption{Functional decomposition at the final checkpoints for $L=1024$. Entries are estimates with pointwise 95\% batch-pair bootstrap intervals. Shares and relative reductions are percentages; contributions are percentage points of whole-model reduction. Values are rounded only for display.}
\label{tab:variance-components-1024}
\footnotesize
\setlength{\tabcolsep}{3pt}
\begin{tabular*}{\linewidth}{@{\extracolsep{\fill}}lrrr@{}}
\toprule
Component & Hard share (\%) & Relative reduction (\%) & Contribution (pp) \\
\midrule
\multicolumn{4}{l}{\textbf{TabICLv2}} \\
TFcol & 86.95 [84.15, 89.29] & 9.05 [3.57, 13.94] & 7.87 [3.07, 12.10] \\
TFrow & 1.62 [1.29, 2.01] & 10.14 [-3.97, 19.36] & 0.16 [-0.07, 0.32] \\
TFicl & 8.23 [6.67, 10.16] & 26.85 [23.55, 29.96] & 2.21 [1.72, 2.85] \\
Prediction head & 3.20 [2.62, 3.94] & 39.90 [32.18, 47.73] & 1.28 [0.94, 1.73] \\
\addlinespace[3pt]
\multicolumn{4}{l}{\textbf{TabSwift}} \\
Row-wise attention & 31.91 [28.55, 35.09] & 20.01 [17.53, 22.73] & 6.38 [5.23, 7.66] \\
MLP & 45.79 [41.89, 50.15] & 19.95 [16.52, 23.60] & 9.13 [7.92, 10.43] \\
Encoders & 12.40 [10.90, 14.10] & 15.56 [11.59, 20.29] & 1.93 [1.56, 2.34] \\
Prediction head & 8.40 [6.72, 10.23] & 37.70 [29.43, 46.28] & 3.17 [2.20, 4.45] \\
Other & 1.51 [1.34, 1.67] & 19.90 [17.17, 22.92] & 0.30 [0.24, 0.37] \\
\addlinespace[3pt]
\multicolumn{4}{l}{\textbf{TabPFN-2.5}} \\
Within-row attention & 18.39 [17.68, 19.00] & 23.33 [20.67, 25.74] & 4.29 [3.74, 4.80] \\
Within-column attention & 50.14 [48.73, 51.58] & 24.08 [22.24, 25.96] & 12.07 [11.07, 13.10] \\
MLP & 17.34 [16.92, 17.74] & 23.19 [21.29, 25.08] & 4.02 [3.67, 4.34] \\
Encoders & 0.21 [0.14, 0.31] & 24.77 [11.80, 34.42] & 0.05 [0.02, 0.10] \\
Prediction head & 13.92 [12.51, 15.47] & 40.28 [33.67, 46.98] & 5.61 [4.36, 7.04] \\
Other & 0.00 [0.00, 0.00] & 3.60 [-2.24, 16.23] & 0.00 [-0.00, 0.00] \\
\addlinespace[3pt]
\bottomrule
\end{tabular*}
\end{table}

\begin{table}[!t]
\centering
\caption{Functional decomposition at the final checkpoints for $L=8192$. Entries are estimates with pointwise 95\% batch-pair bootstrap intervals. Shares and relative reductions are percentages; contributions are percentage points of whole-model reduction. Values are rounded only for display.}
\label{tab:variance-components-8192}
\footnotesize
\setlength{\tabcolsep}{3pt}
\begin{tabular*}{\linewidth}{@{\extracolsep{\fill}}lrrr@{}}
\toprule
Component & Hard share (\%) & Relative reduction (\%) & Contribution (pp) \\
\midrule
\multicolumn{4}{l}{\textbf{TabICLv2}} \\
TFcol & 93.89 [91.93, 95.20] & 5.72 [1.93, 10.24] & 5.37 [1.82, 9.54] \\
TFrow & 1.01 [0.79, 1.29] & 8.35 [5.78, 11.23] & 0.08 [0.06, 0.12] \\
TFicl & 3.88 [3.02, 5.17] & 19.73 [17.02, 22.79] & 0.77 [0.57, 1.07] \\
Prediction head & 1.23 [0.92, 1.71] & 47.23 [40.06, 54.10] & 0.58 [0.40, 0.87] \\
\addlinespace[3pt]
\multicolumn{4}{l}{\textbf{TabSwift}} \\
Row-wise attention & 19.55 [15.25, 24.69] & 10.28 [4.92, 16.72] & 2.01 [1.04, 3.36] \\
MLP & 61.87 [55.81, 66.55] & 8.91 [2.44, 16.41] & 5.51 [1.46, 10.42] \\
Encoders & 13.92 [12.42, 15.90] & 8.02 [2.40, 14.93] & 1.12 [0.33, 2.12] \\
Prediction head & 3.76 [2.76, 5.41] & 38.02 [29.22, 46.81] & 1.43 [0.88, 2.34] \\
Other & 0.90 [0.71, 1.12] & 11.59 [6.52, 17.81] & 0.10 [0.06, 0.17] \\
\addlinespace[3pt]
\multicolumn{4}{l}{\textbf{TabPFN-2.5}} \\
Within-row attention & 17.17 [15.44, 18.97] & 17.23 [13.83, 21.07] & 2.96 [2.43, 3.60] \\
Within-column attention & 56.67 [50.44, 62.21] & 15.30 [10.84, 19.67] & 8.67 [6.06, 11.57] \\
MLP & 14.02 [12.21, 15.63] & 15.93 [12.58, 19.84] & 2.23 [1.78, 2.87] \\
Encoders & 0.25 [0.11, 0.50] & 10.86 [6.37, 15.70] & 0.03 [0.01, 0.05] \\
Prediction head & 11.88 [8.88, 16.14] & 45.56 [36.77, 53.39] & 5.41 [3.72, 7.69] \\
Other & 0.00 [0.00, 0.00] & 7.24 [-6.28, 9.31] & 0.00 [-0.00, 0.00] \\
\addlinespace[3pt]
\bottomrule
\end{tabular*}
\end{table}

\FloatBarrier

\subsection{Budget-sensitivity protocol}
\label{app:beta-sensitivity}

The TabSwift sensitivity runs use fixed task-level budgets $\beta\in\{0.05,0.15,0.30\}$ throughout both stages, following the training protocol in Section~\ref{sec:experiment-setup}.
Each Stage~2 run starts from its own budget's Stage~1 checkpoint.
Model and data seeds are fixed across budgets.
All runs use all-row averaging, the binary cap, selective fallback, and no external noise cap.
The $\beta=0.15$ results reuse the main run.

All budgets use the same 594 TabArena outer splits, default preprocessing, internal validation, and ensemble size 16.
Stage~2 TALENT results use the same 188 datasets as Table~\ref{tab:benchmark-results}, comprising 120 binary and 68 multiclass tasks, with five stratified folds and ensemble size eight.
All 940 folds per condition are observed; checkpoint selection is fixed across datasets.
Binary AUC and multiclass log loss are averaged separately after computing each dataset's fold mean.

Table~\ref{tab:beta-sensitivity-full} reports the full Stage~2 results, including paired win rates.
Runs follow the protocol in Section~\ref{sec:experiment-setup}.

% Generated by scripts/build_experiment_assets.py.
\begin{table}[!t]
\centering
\caption{TabSwift-\method{} entropy-budget sensitivity at Stage~2. AUC and LL denote mean binary ROC-AUC and multiclass log loss; each dataset receives equal weight. Bold marks the best mean for each metric. Win is the percentage of datasets won against Hard, including ties in the denominator. Dataset sets match Table~\ref{tab:benchmark-results}.}
\label{tab:beta-sensitivity-full}
\small
\setlength{\tabcolsep}{5pt}
\begin{tabular}{@{}lrrrrrr@{}}
\toprule
& \multicolumn{3}{c}{TabArena (38)} & \multicolumn{3}{c}{TALENT (188)} \\
\cmidrule(lr){2-4}\cmidrule(l){5-7}
Budget & AUC$\uparrow$ & LL$\downarrow$ & Win$\uparrow$ & AUC$\uparrow$ & LL$\downarrow$ & Win$\uparrow$ \\
\midrule
Hard & 0.8420 & 0.2848 & \textemdash & 0.8713 & 0.4267 & \textemdash \\
$\beta=0.05$ & \textbf{0.8427} & 0.2824 & 68.4 & \textbf{0.8731} & 0.4250 & 64.9 \\
$\beta=0.15$ & 0.8425 & \textbf{0.2811} & 65.8 & 0.8726 & 0.4243 & 67.0 \\
$\beta=0.30$ & 0.8423 & 0.2822 & 60.5 & 0.8727 & \textbf{0.4240} & 63.3 \\
\bottomrule
\end{tabular}
\end{table}

\end{document}